\documentclass[11pt]{article}
\usepackage[T1]{fontenc}
\usepackage[utf8]{inputenc}
\usepackage{lmodern}
\usepackage{microtype}
\usepackage{amsfonts,amssymb,mathtools,bm,mathrsfs}
\usepackage{amsthm}
\usepackage{aliascnt}
\usepackage{geometry}
\usepackage{iftex,ifpdf}
\usepackage[font=footnotesize]{subcaption}
\usepackage[colorlinks=true,allcolors=blue!45!black]{hyperref}
\usepackage[nameinlink,capitalise,noabbrev]{cleveref}
\usepackage{booktabs,array,longtable,tabularx}
\usepackage[numbers,sort&compress]{natbib}
\usepackage{graphicx}
\usepackage{float}
\usepackage{placeins}
\usepackage{xcolor}
\definecolor{clientblue}{HTML}{2F6B9A}
\definecolor{controlorange}{HTML}{B45F06}
\definecolor{biasred}{HTML}{8E3B46}
\definecolor{inkgray}{HTML}{4F5965}
\definecolor{softgray}{HTML}{F2F4F7}
\usepackage{enumitem}
\usepackage{tikz}
\usetikzlibrary{arrows.meta,positioning,fit,calc}
\usepackage{fancyvrb}
\usepackage{url}

\ifpdf
  \DeclareGraphicsExtensions{.pdf,.png,.jpg,.eps}
\else
  \DeclareGraphicsExtensions{.eps}
\fi

\setlist[itemize]{topsep=3pt,itemsep=2pt,parsep=1pt}
\setlist[enumerate]{topsep=3pt,itemsep=2pt,parsep=1pt}

\newcommand{\email}[1]{\href{mailto:#1}{#1}}
\allowdisplaybreaks

\newtheorem{theorem}{Theorem}[section]

\newaliascnt{proposition}{theorem}
\newtheorem{proposition}[proposition]{Proposition}
\aliascntresetthe{proposition}

\newaliascnt{lemma}{theorem}
\newtheorem{lemma}[lemma]{Lemma}
\aliascntresetthe{lemma}

\newaliascnt{corollary}{theorem}
\newtheorem{corollary}[corollary]{Corollary}
\aliascntresetthe{corollary}

\theoremstyle{definition}

\newaliascnt{definition}{theorem}

\aliascntresetthe{definition}

\newaliascnt{remark}{theorem}

\aliascntresetthe{remark}

\newaliascnt{proofnote}{theorem}

\aliascntresetthe{proofnote}

\AddToHook{env/theorem/begin}{\crefalias{section}{theorem}}
\AddToHook{env/proposition/begin}{\crefalias{section}{proposition}}
\AddToHook{env/lemma/begin}{\crefalias{section}{lemma}}
\AddToHook{env/corollary/begin}{\crefalias{section}{corollary}}
\AddToHook{env/definition/begin}{\crefalias{section}{definition}}
\AddToHook{env/remark/begin}{\crefalias{section}{remark}}
\AddToHook{env/proofnote/begin}{\crefalias{section}{proofnote}}

\crefname{assumption}{Assumption}{Assumptions}
\Crefname{assumption}{Assumption}{Assumptions}
\crefname{theorem}{Theorem}{Theorems}
\Crefname{theorem}{Theorem}{Theorems}
\crefname{proposition}{Proposition}{Propositions}
\Crefname{proposition}{Proposition}{Propositions}
\crefname{lemma}{Lemma}{Lemmas}
\Crefname{lemma}{Lemma}{Lemmas}
\crefname{corollary}{Corollary}{Corollaries}
\Crefname{corollary}{Corollary}{Corollaries}
\crefname{remark}{Remark}{Remarks}
\Crefname{remark}{Remark}{Remarks}
\crefname{proofnote}{Proof note}{Proof notes}
\Crefname{proofnote}{Proof note}{Proof notes}

\newcommand{\E}{\mathbb{E}}
\newcommand{\R}{\mathbb{R}}

\newcommand{\Scaf}{\textnormal{\textsc{SCAFFOLD}}}

\newcommand{\eps}{\varepsilon}

\hypersetup{
  colorlinks=true,
  linkcolor=blue!50!black,
  citecolor=blue!50!black,
  urlcolor=blue!50!black
}

\newcommand{\FedAvg}{\textnormal{\textsc{FedAvg}}}

\newcommand{\Sym}{\operatorname{Sym}}

\title{Characterizing Client-Independent Stationary Bias for SCAFFOLD and FedAvg in Federated Learning}
\author{Yi-Ping Tang and Guan-Ju Peng
\thanks{Both authors are with the Graduate Institute of Data Science and Information Computing,
National Chung Hsing University, Taichung, Taiwan.
Email (Yi-Ping Tang): \email{s4109053021@gmail.com};
Email (Guan-Ju Peng): \email{gjpeng@email.nchu.edu.tw};
ORCID (Guan-Ju Peng): \url{https://orcid.org/0000-0001-5508-9485}.}}
\date{}

\ifpdf
\hypersetup{
  hypertexnames=false,
  pdftitle={Characterizing Client-Independent Stationary Bias for SCAFFOLD and FedAvg in Federated Learning},
  pdfauthor={Yi-Ping Tang; Guan-Ju Peng},
  pdfsubject={Higher-order stationary bias in stochastic SCAFFOLD and matched FedAvg},
  pdfkeywords={federated learning, SCAFFOLD, FedAvg, stationary bias, stochastic approximation, control variates}
}
\fi

\begin{document}
\maketitle

\begin{abstract}
Even when all clients optimize the same strongly convex objective, the global model parameters produced by constant-step federated learning may still have a stationary mean that differs from the minimizer. Does increasing the number of clients eliminate this bias, or can some components persist? Answering this question clarifies the limits of client averaging. To this end, we study two federated learning algorithms, \Scaf{} and matched \FedAvg{}, under full participation, a sufficiently smooth objective, and independent bounded additive gradient noise. Analyzing global stationary balance and within-round local fluctuations yields multidimensional expansions with a shared $\gamma/N$ term and explicit client-independent $\gamma^2$ coefficients, where $\gamma$ is the step size and $N$ is the client count. The \Scaf{} coefficient combines the ordinary local-update contribution found in \FedAvg{} with an additional persistent-control contribution. Although zero-sum controls cancel from the direct server increment, they change local second moments, which interact with third-order curvature to shift the stationary mean. For fixed dimension $d$ and local-step count $H$, a remainder of $O_{H,d}(\gamma^2/N+\gamma^3)$, uniform in $N$ for sufficiently small $\gamma$, identifies the second-order coefficients along arbitrary joint small-step, large-client limits after removing the shared leading term. Controlled experiments assess the predicted coefficients. In real-data analyses, finite-tail mean displacement can account for a substantial share of averaged-model parameter error, and a post-hoc comparison shows directional agreement with the derived curvature--noise structure.
\end{abstract}

\section{Introduction}
\label{sec:introduction}

Federated learning enables multiple clients to collaboratively train a shared model while keeping their data distributed. To reduce communication costs, clients may perform several local stochastic-gradient updates before synchronizing their models with a central server. Although clients typically solve the same learning task using a common loss function, differences in their local data distributions can lead to different local optimization objectives and client drift. \Scaf{}, introduced by \citet{karimireddy2020scaffold}, uses control variates to mitigate this drift.

Even when all clients optimize the same objective function, constant-step stochastic updates can produce a systematic parameter displacement. The global model continues to fluctuate, and its expected parameters under the stationary distribution may differ from the objective minimizer. This difference is the \emph{stationary mean bias}. Increasing the number of clients reduces the variance of independent gradient noise averaged at the server, but local updates also change where gradients are evaluated before aggregation. For nonlinear objectives, these local fluctuations can affect both the mean and the variance \citep{glasgow2022sharp}. Therefore, understanding stationary mean bias requires separating contributions that diminish as the number of clients increases from the client-independent effects of local stochastic updates and persistent controls. This distinction also matters for temporal model averaging, which combines successive global models across rounds. Under stationarity, temporal averaging preserves the expected parameter while potentially reducing variance; thus, a small squared-bias share of single-checkpoint error need not remain small for an averaged model.

Recent stationary analyses provide a starting point. In the homogeneous setting, \citet[Theorem~5.3]{mangold2025scaffold} identify a leading \Scaf{} mean-bias contribution proportional to $\gamma/N$, where $\gamma$ is the step size and $N$ the number of clients, and point to higher-order bias not suppressed by client averaging. This expansion and the \FedAvg{} result of \citet[Theorem~3]{mangold2025fedavg} leave the client-independent $\gamma^2$ coefficient within the remainder. Consequently, these results do not identify the algorithm-dependent difference at this order or separate the contribution of ordinary local updates from that of persistent controls. We further compare their precise statements with ours in Section~\ref{sec:related}.

This study compares \Scaf{} and \FedAvg{} under a common strongly convex and sufficiently smooth objective, full participation, and independent bounded additive gradient noise. Matching the objective, noise laws, step size, client count, and number of local updates isolates their algorithmic difference. Our analysis combines global stationary balance with the second moments of within-round local trajectories to identify the persistent-control contribution.

Objective homogeneity also separates stochastic control effects from the correction of client drift. The controls can remain random under stochastic gradients even when all clients share the same objective, so this setting retains the mechanism of interest. Proposition~\ref{prop:affine-shift} further extends the \Scaf{} characterization to clients whose objectives differ by fixed linear terms: translating the controls absorbs these gradient offsets while preserving the global stationary law. In this class, correcting objective-induced drift and retaining stochastic mean bias coexist.

The two algorithms share the same $\gamma/N$ contribution but have different formulas for the client-independent $\gamma^2$ coefficient. Let $B_{20}^{\mathcal A}(H)$ denote this coefficient for algorithm $\mathcal A$, where $H$ is the number of local updates per communication round. The subscript $20$ records the respective powers of $\gamma$ and $1/N$. We obtain the decomposition
\begin{equation}
\label{eq:intro-attribution}
B_{20}^{\Scaf}(H)
=
B_{20}^{\FedAvg}(H)
+
B_{20}^{\mathrm{ctrl}}(H),
\end{equation}
where the \FedAvg{} coefficient is the ordinary local-update contribution and $B_{20}^{\mathrm{ctrl}}(H)$ is the additional persistent-control contribution.

Equation~\eqref{eq:intro-attribution} distinguishes mechanisms that the shared leading term does not resolve. Zero-sum controls cancel from the direct server increment but alter local second moments, whose interaction with third-order curvature produces an additional mean shift. The formulas quantify its curvature--noise geometry and local-step dependence, and the cancellation at aggregation does not imply identical stationary means.

For fixed dimension $d$ and local-step count $H$, we bound the remainder by $O_{H,d}(\gamma^2/N+\gamma^3)$, uniformly in $N$ for sufficiently small $\gamma$. Subtracting the shared $\gamma/N$ term and dividing by $\gamma^2$ leaves an error $O_{H,d}(1/N+\gamma)$. It vanishes along any joint limit $\gamma\to0$, $N\to\infty$, without a relative-rate condition, separating the coefficient from the unresolved remainder.

Our contributions are:
\begin{itemize}
    \item \textbf{Explicit stationary-bias expansions.}
    We derive multidimensional expansions for \Scaf{} and matched
    \FedAvg{}, with explicit $\gamma/N$ and client-independent $\gamma^2$
    coefficients and an $N$-uniform remainder, allowing noncommuting
    Hessian and noise covariance.

    \item \textbf{Mechanism attribution and an exact extension.}
    We separate local-update and persistent-control contributions through
    matched coefficients and local second moments, and extend the
    \Scaf{} result to affine objective shifts.

    \item \textbf{Controlled tests and real-data interpretation.}
    We assess coefficient predictions, contraction cancellation, and
    covariance-matched noise laws in controlled experiments, and examine
    minibatch allocation, model averaging, and post-hoc directional
    agreement in real-data finite-tail studies.
\end{itemize}

\paragraph{Significance and scope.}
The shared leading $\gamma/N$ term cannot distinguish \Scaf{} from matched \FedAvg{}, while existing stationary expansions leave their client-independent $\gamma^2$ coefficients unresolved. Our explicit coefficients and $N$-uniform remainder make this algorithm-dependent difference identifiable in joint small-step, large-client limits. The resulting decomposition quantifies the additional $H$-dependent contribution of persistent controls beyond ordinary local-update bias. By isolating this effect under homogeneous objectives and additive noise, the analysis provides a reference for examining additional interactions in more general federated settings.

\section{Related work}
\label{sec:related}

\paragraph{Federated optimization and stationary-bias characterization.}
\Scaf{} uses control variates to mitigate heterogeneity-induced client drift \citep{karimireddy2020scaffold}. Finite-time analyses study LocalSGD and \Scaf{} under gradient and Hessian similarity, weak convexity, and higher-order smoothness \citep{luo2025revisiting}, and establish sharper \Scaf{} convergence guarantees on quadratic objectives \citep{mangold2025sharper}.

Stochastic local updates can induce bias even with identical client objectives. \citet{glasgow2022sharp} study \emph{iterate bias}, which is the deviation of the expected SGD trajectory from deterministic gradient descent with the same initialization. Their finite-step bounds and continuous perspective connect this bias to nonlinearity and noise under higher-order smoothness assumptions. Our estimand is instead the invariant-law global mean relative to the minimizer.

Stationary distributions provide a complementary perspective on constant-step stochastic optimization \citep{dieuleveut2020bridging,chen2022stationary}. For federated learning, \citet{mangold2025fedavg} derive stationary bias and variance expansions for \FedAvg{}, while \citet{mangold2025scaffold} analyze the joint global-model/control process of \Scaf{}. Table~\ref{tab:prior-bias-resolution} compares the analytical resolution of these results with ours for the same algorithms and stationary-bias quantities.

The comparison specializes the prior results to a common objective and bounded additive noise, with the shared coefficient $B_{11}$ given in Equation~\eqref{eq:scaf-B}. The \FedAvg{} Theorem~3 row additionally assumes identical client noise laws (Assumption~A6); its Theorem~4, specialized to $f_c=f$, allows distinct noise laws and has no heterogeneity contribution.

\begin{table}[htbp]
\centering\small
\caption{Analytical resolution beyond the shared leading bias term. Each entry characterizes $b^{\mathcal A}_{\gamma,N,H}-(\gamma/N)B_{11}$ in the common-objective, additive-noise setting. Prior bounds leave the client-independent second-order coefficient unresolved; the last row identifies it explicitly and bounds the residual.}
\label{tab:prior-bias-resolution}
\begin{tabularx}{\linewidth}{@{}p{0.46\linewidth}X@{}}
\toprule
Analysis & Characterization beyond the shared leading term \\
\midrule
\Scaf{}, \citet[Theorem~5.3]{mangold2025scaffold}
& $O(\gamma^2H+\gamma^{3/2})$ \\
\FedAvg{}, \citet[Theorem~3]{mangold2025fedavg}
& $O(\gamma^2H+\gamma^{3/2})$ \\
\FedAvg{}, \citet[Theorem~4]{mangold2025fedavg}, with $f_c=f$
& $O(\gamma^2H^2+\gamma^{3/2}H)$ \\
This work, Theorems~\ref{thm:scaf-main} and~\ref{thm:fa-main}
& $\gamma^2B_{20}^{\mathcal A}(H)$
  ${}+O_{H,d}(\gamma^2/N+\gamma^3)$ \\
\bottomrule
\end{tabularx}
\end{table}

Dividing the prior remainder bounds by $\gamma^2$ gives $O(H+\gamma^{-1/2})$ or $O(H^2+H\gamma^{-1/2})$, so neither identifies the second-order coefficient. Under the additional assumptions of Section~\ref{sec:setting}, our normalized remainder is $O_{H,d}(1/N+\gamma)$, uniformly in $N$ for fixed $H,d$, and vanishes in every joint limit $\gamma\to0$, $N\to\infty$. The explicit coefficients thus identify the direction, magnitude, and local-step dependence of the client-independent contribution, separating ordinary local updates from the additional persistent-control effect. 

\citet[Lemma~5.1]{mangold2025scaffold} already identify the leading
stationary control covariance, $(1-1/N)C/H+O_H(\gamma)$ in the
common-objective, common-covariance setting. Our analysis propagates
this control-fluctuation term through local second moments into the
matched stationary mean coefficients, with a common $N$-uniform
remainder.

\paragraph{Mechanisms and attribution of stationary bias.}
Related work has examined distinct sources of bias:
\citet{allmeier2024bias} derive a Lyapunov-based first-order expression for time-averaged bias under Markovian noise; \citet{huo2024collusion} separate memory, nonlinearity, and their interaction in nonlinear stochastic approximation; and \citet{versini2026dsgd} distinguish stochasticity and decentralization in decentralized SGD. Our analysis instead quantifies the additional stationary mean-bias coefficient attributable to persistent controls through a matched \Scaf{}--\FedAvg{} comparison.

\section{Setting and preliminaries}
\label{sec:setting}

We study \Scaf{} and \FedAvg{} in a matched, full-participation setting with identical client objectives and bounded additive gradient noise. Both algorithms use the same client noise laws and local-update configuration; the key difference is that \Scaf{} maintains persistent client control vectors.

\subsection{Objective and stochastic-gradient model}

Let $x\in\R^d$ denote the global model parameter vector. All $N\geq2$ clients optimize the same objective $f:\R^d\to\R$ and participate in every communication round. Each client takes $H\geq2$ local steps of size $\gamma>0$ per round. The dimension $d$ and local-step count $H$ are fixed, while $N$ and $\gamma$ may vary.

We assume $f\in C^5(\R^d)$, with uniformly bounded third-, fourth-, and fifth-order derivatives, and fixed constants $0<\mu\leq L<\infty$ such that
\begin{equation}
\label{eq:setting-regularity}
\mu I_d\preceq\nabla^2 f(x)\preceq LI_d
\qquad\text{for every }x\in\R^d.
\end{equation}
Thus $f$ has a unique minimizer $x^\star$. We translate it to the origin, so $x^\star=0$ and $\nabla f(0)=0$.

At local step $h$ of round $t$, client $c$ uses the stochastic-gradient oracle
\begin{equation}
\label{eq:additive-oracle}
\nabla F_{t,c,h}(y)=\nabla f(y)+\eps_{t,c,h},\qquad
\E[\eps_{t,c,h}]=0,\qquad
\E[\eps_{t,c,h}\eps_{t,c,h}^{\top}]=C\succeq0.
\end{equation}
For each fixed $N$, client $c$ has a noise law $\nu_{N,c}$ that is unchanged across rounds and local steps and does not depend on $\gamma$. Noise vectors are mutually independent across clients, rounds, and local steps; each round's fresh noise array is independent of the complete round-start state.

The laws may differ across clients, and coordinates within a noise vector may be correlated. All clients nevertheless share the covariance $C$ and the almost-sure bound $\|\eps_{t,c,h}\|\leq K_\eps$, where $C$ and $K_\eps$ are independent of $N$ and $\gamma$.

We define
\begin{equation}
\label{eq:ATC}
A:=\nabla^2f(0)=A^{\top}\succ0,
\qquad T:=\nabla^3f(0).
\end{equation}
Then, for any symmetric matrix $M\in\R^{d\times d}$, the vector $T[M]\in\R^d$ is the contraction
\begin{equation}
\label{eq:tensor-contraction}
(T[M])_i:=\sum_{j,k=1}^{d}T_{ijk}M_{jk}.
\end{equation}
Finally, by letting $S_C$ denote the unique symmetric solution of the Lyapunov equation, we have
\begin{equation}
\label{eq:lyapunov}
AS_C+S_CA=C,
\end{equation}
which exists because $A\succ0$. These quantities describe the local curvature, gradient noise, and leading stationary fluctuations appearing in the main results.

\subsection{SCAFFOLD updates}

At the start of a round, the server holds $x$ and client $c$ has a control vector $\xi_c$. We use the zero-sum parametrization
\begin{equation}
\label{eq:zero-sum}
\sum_{c=1}^{N}\xi_c=0.
\end{equation}
Here, $\xi_c$ is the correction added to client $c$'s stochastic gradient. It remains fixed during the $H$ local steps and is updated only after server aggregation. Starting from $\theta_{c,0}=x$, each client updates using
\begin{equation}
\label{eq:scaf-local}
\theta_{c,h+1}=\theta_{c,h}
-\gamma\bigl(\nabla f(\theta_{c,h})+\xi_c+\eps_{c,h+1}\bigr),
\qquad h=0,\ldots,H-1.
\end{equation}
We suppress the round index when considering one round. The server averages the final local models and refreshes the controls:
\begin{equation}
\label{eq:scaf-server}
x^+=\frac1N\sum_{c=1}^{N}\theta_{c,H},\qquad
\xi_c^+=\xi_c+\frac{\theta_{c,H}-x^+}{\gamma H},
\end{equation}
where $+$ denotes the next round-start state. Equation~\eqref{eq:scaf-server} preserves the zero-sum condition in Equation~\eqref{eq:zero-sum} exactly. This is the full-participation zero-sum representation of SCAFFOLD used by \citet{mangold2025scaffold}.

\subsection{Matched FedAvg updates}

For comparison, the matched \FedAvg{} algorithm uses the same objective, client noise laws, $N$, $H$, and $\gamma$, but maintains no control state. Matching concerns the stochastic setting and local-update configuration, not necessarily identical noise realizations. Starting from $\theta^{\FedAvg}_{c,0}=x$, its local and server updates are
\begin{equation}
\label{eq:fa-local}
\theta^{\FedAvg}_{c,h+1}
=\theta^{\FedAvg}_{c,h}-\gamma\bigl(
\nabla f(\theta^{\FedAvg}_{c,h})+\eps_{c,h+1}\bigr),
\quad h=0,\ldots,H-1,
\end{equation}
\begin{equation}
\label{eq:fa-server}
x^+=\frac1N\sum_{c=1}^{N}\theta^{\FedAvg}_{c,H}.
\end{equation}

\subsection{Stationary mean bias and coefficient notation}

For sufficiently small $\gamma$, the two round-update chains admit unique invariant laws \citep{mangold2025scaffold,mangold2025fedavg}. Denote by $\pi^{\Scaf}_{\gamma,N,H}$ the invariant law of the joint global-model/control state, and by $\pi^{\FedAvg}_{\gamma,N,H}$ that of the \FedAvg{} global model. We study the stationary mean biases of the \emph{global model parameters}:
\begin{align}
\label{eq:scaf-bias-def}
b^{\Scaf}_{\gamma,N,H}
&:=\E_{\pi^{\Scaf}_{\gamma,N,H}}[x]-x^\star
=\E_{\pi^{\Scaf}_{\gamma,N,H}}[x],\\
\label{eq:fa-bias-def}
b^{\FedAvg}_{\gamma,N,H}
&:=\E_{\pi^{\FedAvg}_{\gamma,N,H}}[x]-x^\star
=\E_{\pi^{\FedAvg}_{\gamma,N,H}}[x].
\end{align}
Unless stated otherwise, within-round expectations include both the stationary round-start state and the fresh noise of that round.

We index the coefficients in our expansions by powers of $\gamma$ and $N^{-1}$. Specifically, $B_{11}$ denotes the $\gamma/N$ coefficient, while $B_{20}^{\mathcal A}(H)$ denotes the client-independent $\gamma^2$ coefficient for $\mathcal A\in\{\Scaf,\FedAvg\}$. Since $d$ is fixed, we omit a dimension superscript. All results below hold for $0<\gamma\leq\gamma_0$, where $\gamma_0>0$ may depend on the fixed objective/noise regularity constants, $H$, and $d$, but not on $N$; all stated $O_{H,d}(\cdot)$ constants are uniform over $N\geq2$ in this regime.

\section{Main results}
\label{sec:main-results}

The main results resolve two contributions to stationary mean bias: a client-averaged term proportional to $\gamma/N$ and a client-independent term proportional to $\gamma^2$. We first establish the \Scaf{} expansion, then compare it with matched \FedAvg{} to separate the effect of local stochastic updates from that of persistent controls. Throughout this section, $d$ and $H\geq2$ are fixed, and the assumptions of Section~\ref{sec:setting} hold. The proof details for the theorems, corollaries, and propositions in this section are provided in Appendix \ref{app:scaf-proof} and \ref{app:fa-proof}.

\subsection{Two bias scales in SCAFFOLD}

Server averaging reduces fluctuations across independent clients, but nonlinear effects accumulated within local trajectories need not vanish as $N$ increases. The following theorem makes the two resolved bias scales explicit.

\begin{theorem}[Stationary mean-bias expansion for SCAFFOLD]
\label{thm:scaf-main}
Under the assumptions of Section~\ref{sec:setting} and the \Scaf{} updates in Equations~\eqref{eq:scaf-local} and~\eqref{eq:scaf-server}, there exist constants $\gamma_0>0$ and $C^{\Scaf}_{H,d}<\infty$, independent of $N$, such that for every $N\geq2$ and $0<\gamma\leq\gamma_0$,
\begin{equation}
\label{eq:scaf-main-expansion}
b^{\Scaf}_{\gamma,N,H}
=-\frac{\gamma}{2N}A^{-1}T[S_C]
-\gamma^2\frac{(H-1)(5H-1)}{12H}A^{-1}T[C]
+R^{\Scaf}_{\gamma,N,H},
\end{equation}
where
\begin{equation}
\label{eq:scaf-main-rem}
\bigl\|R^{\Scaf}_{\gamma,N,H}\bigr\|
\leq C^{\Scaf}_{H,d}\left(\frac{\gamma^2}{N}+\gamma^3\right).
\end{equation}
\end{theorem}

The expansion distinguishes a term proportional to $\gamma/N$ from one proportional to $\gamma^2$ with an $N$-independent coefficient:
\begin{equation}
\label{eq:scaf-B}
B_{11}=-\frac12A^{-1}T[S_C],\qquad
B_{20}^{\Scaf}(H)
=-\frac{(H-1)(5H-1)}{12H}A^{-1}T[C].
\end{equation}
Here, $B_{11}$ multiplies $\gamma/N$, whereas $B_{20}^{\Scaf}(H)$ multiplies $\gamma^2$ without a factor $1/N$. The latter is governed by $A^{-1}T[C]$: the noise covariance $C$ interacts with the third derivative $T$, and $A^{-1}$ maps the resulting vector to a parameter displacement. In particular, $T[C]=0$ eliminates this coefficient even when the gradient noise is nonzero.

At these resolved orders, increasing $N$ suppresses the $\gamma/N$ term but not the explicit $\gamma^2$ term; reducing $\gamma$ affects them linearly and quadratically, respectively. For either algorithm $\mathcal A\in\{\Scaf,\FedAvg\}$ with $B_{11}\ne0$, the ratio of their magnitudes is
\begin{equation}
\label{eq:bias-scale-ratio}
\frac{\|\gamma^2B_{20}^{\mathcal A}(H)\|}
     {\|(\gamma/N)B_{11}\|}
=\gamma N\,\frac{\|B_{20}^{\mathcal A}(H)\|}{\|B_{11}\|}.
\end{equation}
Equation~\eqref{eq:bias-scale-ratio} shows that the second-order contribution need not be smaller when $N$ is large. The full stationary bias is the sum of both vectors and the remainder, so cancellation is possible. Their relative magnitudes also do not, by themselves, determine the squared-bias share of parameter mean-squared error.

\subsection{Local-update bias without controls: matched FedAvg}

The \Scaf{} coefficient alone does not reveal which part of the $\gamma^2$ term arises from local updates common to both algorithms. Matched \FedAvg{} provides the corresponding expansion without a persistent control state.

\begin{theorem}[Stationary mean-bias expansion for matched FedAvg]
\label{thm:fa-main}
Under the assumptions of Section~\ref{sec:setting} and the matched \FedAvg{} updates in Equations~\eqref{eq:fa-local} and~\eqref{eq:fa-server}, the step-size threshold $\gamma_0$ can be chosen jointly with Theorem~\ref{thm:scaf-main}. There exists $C^{\FedAvg}_{H,d}<\infty$, independent of $N$ and $\gamma$, such that for $N\geq2$ and $0<\gamma\leq\gamma_0$,
\begin{equation}
\label{eq:fa-main-expansion}
b^{\FedAvg}_{\gamma,N,H}
=-\frac{\gamma}{2N}A^{-1}T[S_C]
-\gamma^2\frac{H-1}{4}A^{-1}T[C]
+R^{\FedAvg}_{\gamma,N,H},
\end{equation}
where
\begin{equation}
\label{eq:fa-main-rem}
\bigl\|R^{\FedAvg}_{\gamma,N,H}\bigr\|
\leq C^{\FedAvg}_{H,d}\left(\frac{\gamma^2}{N}+\gamma^3\right).
\end{equation}
In particular,
\begin{equation}
\label{eq:fa-B}
B_{20}^{\FedAvg}(H)=-\frac{H-1}{4}A^{-1}T[C].
\end{equation}
\end{theorem}

The two expansions share $B_{11}$ but have different $\gamma^2$ coefficients. Thus, when $T[C]\ne0$, local stochastic updates produce a client-independent second-order mean shift even without controls. Subtracting the coefficients then identifies the additional contribution associated with \Scaf{} controls.

\subsection{The additional contribution of persistent controls}

Define the difference between the two second-order coefficients as
\begin{equation}
\label{eq:ctrl-def}
B_{20}^{\mathrm{ctrl}}(H)
:=-\frac{(H-1)(2H-1)}{12H}A^{-1}T[C].
\end{equation}

\begin{corollary}[Matched SCAFFOLD--FedAvg attribution]
\label{cor:attribution}
Under the assumptions of Theorems~\ref{thm:scaf-main} and~\ref{thm:fa-main}, we derive
\begin{equation}
\label{eq:matched-control}
\boxed{B_{20}^{\Scaf}(H)
=B_{20}^{\FedAvg}(H)+B_{20}^{\mathrm{ctrl}}(H).}
\end{equation}
Consequently, the following equation holds:
\begin{equation}
\label{eq:matched-bias-difference}
b^{\Scaf}_{\gamma,N,H}-b^{\FedAvg}_{\gamma,N,H}
=\gamma^2B_{20}^{\mathrm{ctrl}}(H)
+O_{H,d}\!\left(\frac{\gamma^2}{N}+\gamma^3\right).
\end{equation}
\end{corollary}

For $H\geq2$, the \FedAvg{} and control coefficients are positive scalar multiples of the same vector $-A^{-1}T[C]$ and therefore point in the same direction whenever nonzero. Specifically,
$B_{20}^{\mathrm{ctrl}}(H)=\frac{2H-1}{3H}B_{20}^{\FedAvg}(H)$,
so at $H=2$ the control coefficient equals half the \FedAvg{} coefficient. Equation~\eqref{eq:ctrl-def} also yields
$B_{20}^{\mathrm{ctrl}}(4)=\tfrac72 B_{20}^{\mathrm{ctrl}}(2)$,
a factor of $3.5$ between these two fixed-$H$ cases.

The leading coefficient $B_{11}$ has no $H$ dependence, whereas the second-order prefactors depend on $H$ and the remainder constants and admissible step-size threshold may also depend on it. These are comparisons of fixed local-step counts, not a uniform growing-$H$ regime. Section~\ref{sec:mechanism} explains the underlying control-state mechanism.

\subsection{Identifying the second-order coefficients}

Identifying a $\gamma^2$ coefficient requires a remainder that vanishes after normalization by $\gamma^2$. For \Scaf{}, we subtract the shared leading term and define
\begin{equation}
\label{eq:scaf-normalized}
\mathcal C^{\Scaf}_{\gamma,N,H}
:=\frac{b^{\Scaf}_{\gamma,N,H}-(\gamma/N)B_{11}}{\gamma^2}.
\end{equation}
Theorem~\ref{thm:scaf-main} implies
\begin{equation}
\label{eq:scaf-normalized-limit}
\mathcal C^{\Scaf}_{\gamma,N,H}
=B_{20}^{\Scaf}(H)+O_{H,d}\!\left(\frac1N+\gamma\right).
\end{equation}
Because the remainder bound is uniform in $N$, the normalized error vanishes along \emph{any} joint limit $\gamma\to0$, $N\to\infty$, without a condition relating their rates. The same holds for \FedAvg{} and the normalized matched difference. This identifies both algorithm-specific coefficients and the additional control contribution at a common asymptotic resolution.

\subsection{An exact extension to affine objective heterogeneity}
\label{sec:affine-extension}

The homogeneous result extends exactly to a restricted form of objective heterogeneity: clients may have different local minimizers while sharing the same curvature. The key is to absorb constant gradient offsets into the \Scaf{} control state.

\begin{samepage}
\begin{proposition}[Control translation for affine objective shifts]
\label{prop:affine-shift}
Let $f$ and the additive-noise laws satisfy Section~\ref{sec:setting}, and fix $N\ge2$, $H\ge2$ and $0<\gamma\le\gamma_0$. For deterministic vectors $a_c\in\R^d$ with $\sum_ca_c=0$, we set
\begin{equation}
\label{eq:affine-objectives}
f_c(x)=f(x)+a_c^\top x,\qquad \widetilde\xi_c=\xi_c+a_c.
\end{equation}
We run SCAFFOLD on the zero-sum control state space, with $\nabla f_c$ in Equation~\eqref{eq:scaf-local} and the same server/control updates in Equation~\eqref{eq:scaf-server}. The transformed state $(x,\widetilde\xi)$ follows exactly the homogeneous SCAFFOLD recursion for $f$. Its global invariant distribution therefore coincides with that of the homogeneous chain. In particular, the expansion and remainder in Theorem~\ref{thm:scaf-main} remain valid for its global stationary mean bias.
\end{proposition}
\end{samepage}

Although individual clients may have different minimizers, their average objective remains $N^{-1}\sum_c f_c=f$. With translated initial controls and shared noise realizations, the affine and homogeneous \Scaf{} recursions coincide pathwise, although unrelated control initializations need not yield identical transients. The state translation therefore preserves the global stationary distribution and its mean-bias expansion, even in the presence of these client-specific affine shifts.

\section{Why persistent controls affect stationary bias}
\label{sec:mechanism}

Corollary~\ref{cor:attribution} identifies an additional client-independent bias coefficient in \Scaf{} relative to matched \FedAvg{}. The effect is indirect: zero-sum controls cancel from the server update, but they change the local trajectories before aggregation. We trace this effect through the server increment, local second moments, and stationary mean balance.

\subsection{Server-level cancellation does not imply identical local trajectories}

Let $\bar\eps_{h+1}:=N^{-1}\sum_c\eps_{c,h+1}$. Summing the local updates in Equation~\eqref{eq:scaf-local} and averaging over clients, with the zero-sum condition in Equation~\eqref{eq:zero-sum}, gives
\begin{equation}
\label{eq:mechanism-global-increment}
x^+-x
=-\gamma\sum_{h=0}^{H-1}\frac1N\sum_{c=1}^N\nabla f(\theta_{c,h})
-\gamma\sum_{h=0}^{H-1}\bar\eps_{h+1}.
\end{equation}
The controls disappear explicitly from this expression, but they still determine the local iterates $\theta_{c,h}$ at which the gradients are evaluated. Matched \FedAvg{} has the same form of server increment, with different local trajectories because it has no control state.

\subsection{Controls create an additional local second-moment contribution}

To measure the effect on local trajectories, define the stationary \emph{raw} second moments $M_x^{\mathcal A}=\E[xx^\top]$ and
$M_h^{\mathcal A}=N^{-1}\sum_c\E[\theta_{c,h}\theta_{c,h}^\top]$ for algorithm $\mathcal A$. Their within-round difference is
\begin{equation}
\label{eq:mechanism-moments}
U_h^{\mathcal A}:=M_h^{\mathcal A}-M_x^{\mathcal A},
\qquad \mathcal A\in\{\Scaf,\FedAvg\}.
\end{equation}
For $h=0,\ldots,H-1$, the moment estimates are
\begin{align}
\label{eq:mechanism-U-fa}
U_h^{\FedAvg}
&=\gamma^2hC+O_{H,d}\!\left(\frac{\gamma^2}{N}+\gamma^3\right),\\
\label{eq:mechanism-U-scaf}
U_h^{\Scaf}
&=\gamma^2\left(h+\frac{h^2}{H}\right)C
  +O_{H,d}\!\left(\frac{\gamma^2}{N}+\gamma^3\right).
\end{align}
The shared $hC$ term reflects fresh gradient noise accumulated over local steps. In \Scaf{}, the same round-start control acts throughout those steps, adding the $h^2C/H$ term. Its quadratic dependence on $h$ reflects this persistence; the $1/H$ factor comes from the leading stationary control covariance established by \citet[Lemma~5.1]{mangold2025scaffold}.

\subsection{Local fluctuations are converted into a mean shift by curvature}

Taking expectations of the server increment under stationarity yields the balance equation
\begin{equation}
\label{eq:mechanism-balance}
0=\sum_{h=0}^{H-1}\frac1N\sum_{c=1}^N\E[\nabla f(\theta_{c,h})].
\end{equation}
Because $f$ is nonlinear, these expected gradients depend not only on the mean parameter but also on local fluctuations. A Taylor expansion around the minimizer makes the role of the second moments explicit:
\begin{equation}
\label{eq:mechanism-master}
b^{\mathcal A}_{\gamma,N,H}
=-\frac12A^{-1}T[M_x^{\mathcal A}]
-\frac1{2H}A^{-1}T\!\left[\sum_{h=0}^{H-1}U_h^{\mathcal A}\right]
+O_{H,d}\!\left(\frac{\gamma^2}{N}+\gamma^3\right).
\end{equation}
Both algorithms have the same leading global second moment,
$M_x^{\mathcal A}=(\gamma/N)S_C+O_{H,d}(\gamma^2/N+\gamma^3)$.
Their additional client-independent difference therefore comes from the local moments in Equations~\eqref{eq:mechanism-U-fa} and~\eqref{eq:mechanism-U-scaf}. Inserting the extra $\gamma^2h^2C/H$ term into Equation~\eqref{eq:mechanism-master} gives
\[
-\frac{\gamma^2}{2H}A^{-1}T\!\left[\sum_{h=0}^{H-1}\frac{h^2}{H}C\right]
=\gamma^2B_{20}^{\mathrm{ctrl}}(H),
\]
which recovers the additional coefficient in Corollary~\ref{cor:attribution}.

Thus the controls affect the stationary mean through their influence on local fluctuations, even though their direct contribution cancels at the server. The nonlinear curvature converts the difference in local second moments into the matched bias difference.

\section{Numerical study}
\label{sec:numerics}

\setcounter{topnumber}{3}
\setcounter{bottomnumber}{2}
\setcounter{totalnumber}{4}
\renewcommand{\topfraction}{0.90}
\renewcommand{\bottomfraction}{0.80}
\renewcommand{\textfraction}{0.08}
\renewcommand{\floatpagefraction}{0.75}

We examine the coefficient formulas, their underlying mechanism, and their
implications for averaged models. Controlled additive-noise experiments
isolate the individual coefficients, the extra control contribution, and
its dependence on noise geometry. Shared-objective logistic-regression
experiments examine how client allocation and temporal averaging affect
finite-tail mean displacement. A final post-hoc comparison connects the
observed displacement direction to the analytical curvature--noise vector.

\subsection{Evaluation quantities and experimental settings}
\label{sec:num-protocol}

The controlled experiments use nonlinear strongly convex objectives with
known minimizers and coefficient targets. We estimate the global mean with
the stationary mean-balance identity. Subtracting the shared leading term
and rescaling exposes the second-order coefficient:
\begin{equation}
\label{eq:num-C}
\widehat{\mathcal C}^{\mathcal A}_j
:=\frac{\widehat b^{\mathcal A}_j-(\gamma/N)B_{11,j}}{\gamma^2},
\qquad j=1,2.
\end{equation}
Theorems~\ref{thm:scaf-main} and~\ref{thm:fa-main} predict convergence
to $B_{20}^{\mathcal A}(H)$ as the step size decreases and the client
count increases. The fitted intercept of a linear extrapolation estimates
this limiting coefficient. More details regarding the experimental settings are provided in appendices. Table~\ref{tab:num-controlled-settings}
collects the controlled parameter grids; Tables~\ref{tab:num-run-schedules}
and~\ref{tab:num-inference} give run lengths and interval procedures.
The real-data settings are collected in Table~\ref{tab:num-real-settings}.

\subsection{Controlled tests of analytical predictions}
\label{sec:num-controlled}

We first estimate each algorithm's coefficient and then directly compare
the algorithms. Complementary tests separate client-count dependence,
multidimensional consistency, local-step dependence, and affine-shift
equivalence. The final controlled experiment tests two specific predictions:
cancellation when $T[C]=0$ and invariance of the matched coefficient under
covariance-preserving changes to the noise law.

\subsubsection{Individual stationary-bias coefficients}
\label{sec:num-coefficients}

\paragraph{Design and rationale.}
This experiment checks the explicit $B_{20}^{\Scaf}(H)$ and
$B_{20}^{\FedAvg}(H)$ formulas separately. A two-dimensional objective
with noncommuting Hessian and noise covariance provides a concrete test of
the matrix-valued geometry. Along the joint path $\gamma=1/N$, both
normalized remainder scales decrease together. A smaller-step grid probes
the stability of the $H=4$ estimates.

\paragraph{Results and interpretation.}
Table~\ref{tab:num-attribution-summary} compares the fitted vectors with
the analytical targets. For $H=2$, both algorithms' targets fall within
their simultaneous intervals. For $H=4$, the \FedAvg{} intervals contain
the targets on both grids; the \Scaf{} estimates move closer on the
smaller-step grid, whose intervals also contain both targets. These
comparisons support the predicted coefficient values while showing greater
grid sensitivity for \Scaf{} at $H=4$. Complete interval endpoints are in
Table~\ref{tab:num-core}. The matched-pair rows are interpreted next.

\begin{table}[!htbp]
\centering\small
\setlength{\tabcolsep}{3pt}
\caption{Analytical coefficient vectors and fitted finite-grid intercepts.
Individual rows target $B_{20}^{\mathcal A}(H)$; matched-pair rows target
$B_{20}^{\mathrm{ctrl}}(H)$. P and S denote the primary and smaller-step
joint grids in Table~\ref{tab:num-controlled-settings}. Values are rounded
to seven decimal places; simultaneous intervals are in
Table~\ref{tab:num-core}.}
\label{tab:num-attribution-summary}
\begin{tabular}{@{}lccrr@{}}
\toprule
Study & $H$ & Grid & Target vector & Estimated intercept vector \\
\midrule
\Scaf{} & 2 & P & $(-0.0480040,\,-0.0600050)$ & $(-0.0476713,\,-0.0595811)$ \\
\Scaf{} & 4 & P & $(-0.1520127,\,-0.1900159)$ & $(-0.1502614,\,-0.1874524)$ \\
\Scaf{} & 4 & S & $(-0.1520127,\,-0.1900159)$ & $(-0.1514772,\,-0.1893273)$ \\
\addlinespace
\FedAvg{} & 2 & P & $(-0.0320027,\,-0.0400033)$ & $(-0.0320430,\,-0.0400539)$ \\
\FedAvg{} & 4 & P & $(-0.0960080,\,-0.1200100)$ & $(-0.0959717,\,-0.1198059)$ \\
\FedAvg{} & 4 & S & $(-0.0960080,\,-0.1200100)$ & $(-0.0959727,\,-0.1199590)$ \\
\addlinespace
Paired difference & 2 & P & $(-0.0160013,\,-0.0200017)$ & $(-0.0160363,\,-0.0200366)$ \\
Paired difference & 4 & P & $(-0.0560047,\,-0.0700059)$ & $(-0.0543238,\,-0.0676943)$ \\
Paired difference & 4 & S & $(-0.0560047,\,-0.0700059)$ & $(-0.0559262,\,-0.0698957)$ \\
\bottomrule
\end{tabular}
\end{table}

\subsubsection{Matched contribution of persistent controls}
\label{sec:num-attribution}

\paragraph{Design and rationale.}
To isolate the additional term in Corollary~\ref{cor:attribution},
we compare \Scaf{} and \FedAvg{} using the same initial global model
and indexed noise within each pair. Their shared leading coefficient
cancels in
\begin{equation}
\label{eq:num-D}
\widehat D_j
:=\widehat{\mathcal C}^{\Scaf}_j-\widehat{\mathcal C}^{\FedAvg}_j
=\frac{\widehat b_j^{\Scaf}-\widehat b_j^{\FedAvg}}{\gamma^2}.
\end{equation}
The target is $B_{20}^{\mathrm{ctrl}}(H)$. Common noise reduces sampling
variation in the comparison, making the extra contribution easier to
resolve than by subtracting independently estimated algorithm means.

\paragraph{Results and interpretation.}
Figure~\ref{fig:num-matched-primary} displays the matched estimates and
analytical targets. The $H=2$ intercept intervals contain both targets.
For $H=4$, the smaller-step grid reduces absolute intercept errors by about
$95\%$, although its intervals still exclude the targets. The experiment
therefore resolves a nonzero control-associated contribution and shows
improved coefficient agreement on the finer grid, with a remaining
$H=4$ offset. Appendix~\ref{app:num-h4-diagnostics} examines this offset:
the grid-point residuals scale approximately with $\gamma+1/N$, and their
finite-step, finite-client, and finite-run sources remain unresolved.

\begin{figure}[!htbp]
\centering
\captionsetup[subfigure]{font=small}
\begin{subfigure}[t]{0.49\textwidth}
\centering
\includegraphics[width=\linewidth]{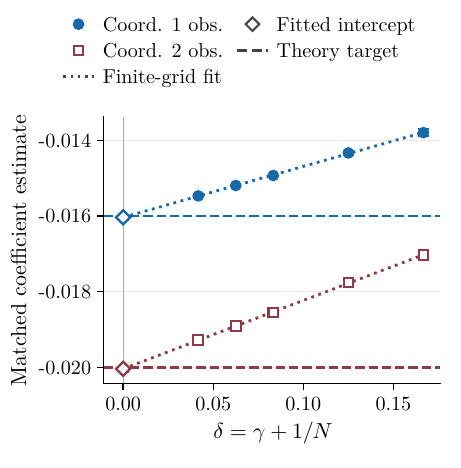}
\caption{Two local steps}
\end{subfigure}\hfill
\begin{subfigure}[t]{0.49\textwidth}
\centering
\includegraphics[width=\linewidth]{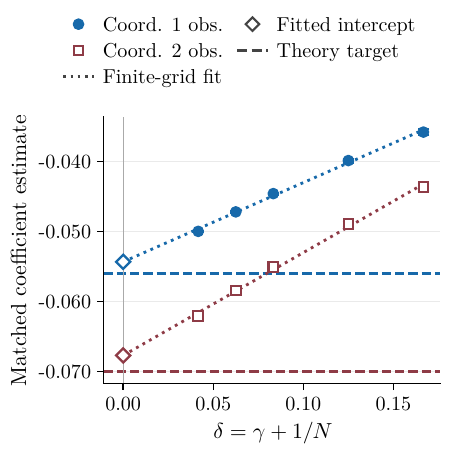}
\caption{Four local steps}
\end{subfigure}
\caption{Matched \Scaf{}--\FedAvg{} coefficient estimates for (a) $H=2$
and (b) $H=4$ on the primary joint grids. The ordinate is
$\widehat D_j$ from Equation~\eqref{eq:num-D}. Circles and squares show
the two coordinates, dotted lines show finite-grid linear fits, and
dashed horizontal lines give the analytical targets. Open diamonds at
$\delta=0$ mark the fitted intercepts.
Pointwise bars use simultaneous Student-$t$ intervals; intercept bars
use simultaneous seed-block-bootstrap intervals.}
\label{fig:num-matched-primary}
\end{figure}

\subsubsection{Client-count dependence at fixed step size}
\label{sec:num-supporting}

\paragraph{Design and rationale.}
The joint-grid experiment changes $N$ and $\gamma$ together. Here, we hold
the step size and local-step count fixed and increase $N$, isolating the
effect of additional client averaging. A client-independent second-order
term should remain visible as finite-client corrections to the normalized
statistic decrease.

\paragraph{Results and interpretation.}
The normalized \Scaf{} contribution remains visibly nonzero across the
tested client counts (Figure~\ref{fig:E-p2-plateau}). The fitted offsets
from the joint-limit coefficient are reported in
Table~\ref{tab:E-plateau}; the fixed-step expansion permits an
$O(\gamma)$ offset. This result illustrates the distinction between
increasing the client count and reducing the step size: additional clients
suppress the client-dependent term while leaving a residual mean shift.

\begin{figure}[!htbp]
\centering
\captionsetup[subfigure]{font=small}
\begin{subfigure}[t]{0.49\textwidth}
\centering
\includegraphics[width=\linewidth]{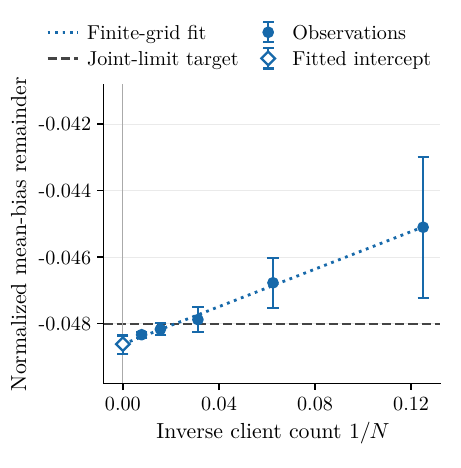}
\caption{Coordinate 1}
\end{subfigure}\hfill
\begin{subfigure}[t]{0.49\textwidth}
\centering
\includegraphics[width=\linewidth]{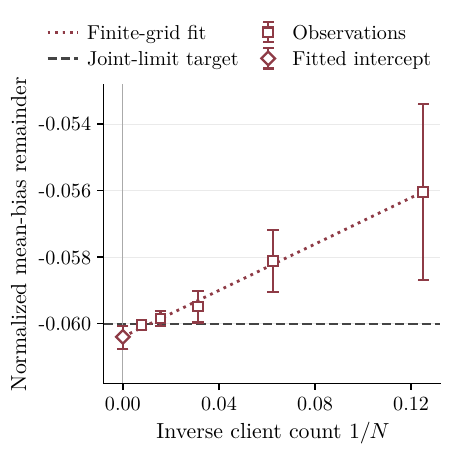}
\caption{Coordinate 2}
\end{subfigure}
\caption{Client-count dependence at fixed step size, shown by coordinate.
The ordinate is $\widehat{\mathcal C}^{\Scaf}_j$ from
Equation~\eqref{eq:num-C}: the mean-bias estimate after subtracting the
shared $\gamma/N$ term and dividing by $\gamma^2$. Dotted lines give
finite-grid fits and open diamonds their intercepts at $1/N=0$.
Dashed horizontal lines give the joint-limit references
$B_{20}^{\Scaf}(2)$; a fixed step size permits an $O(\gamma)$ offset.
Bars show simultaneous Student-$t$ intervals at grid points and
seed-block-bootstrap intervals for fitted intercepts.}
\label{fig:E-p2-plateau}
\end{figure}

\subsubsection{Multidimensional and coordinate consistency}
\label{sec:num-multidimensional}

\paragraph{Design and rationale.}
A five-dimensional \Scaf{} instance tests the coefficient formula beyond
the primary two-dimensional example, again using noncommuting Hessian and
covariance matrices. A separate rotation check transforms the objective,
noise, and initial state together. The resulting trajectories and
coefficients should transform by the same orthogonal map.

\paragraph{Results and interpretation.}
All ten coordinate targets across $H=2$ and $H=4$ lie within the reported
simultaneous intervals (Table~\ref{tab:E-d5}). The rotation check has
pathwise residuals below $1.6\times10^{-14}$
(Appendix~\ref{app:equivariance-check}). The five-dimensional results
support the coefficient formula across coordinates; the rotation check
confirms that the implementation respects its coordinate transformation.

\subsubsection{Local-step dependence on an independent grid}
\label{sec:num-supp-theory}

\paragraph{Design and rationale.}
The theory assigns different roles to local training length: $B_{11}$ is
shared across $H$, whereas $B_{20}^{\mathrm{ctrl}}(H)$ changes explicitly
with $H$. An independent implementation tests both predictions. Individual
fits of $(N/\gamma)\widehat b$ estimate $B_{11}$ freely. Matched fits on
a Cartesian grid vary $\gamma$ and $1/N$ independently, allowing a direct
comparison with the predicted relation
$B_{20}^{\mathrm{ctrl}}(4)=3.5B_{20}^{\mathrm{ctrl}}(2)$.

\paragraph{Results and interpretation.}
The leading-coefficient intervals contain their targets for $H=1,2,4$;
both algorithms have offsets at $H=8$. The adjusted \Scaf{} $H=8$
versus $H=2$ leading-intercept contrasts exclude zero
(Table~\ref{tab:supp-leading}). For the matched coefficient,
Table~\ref{tab:supp-control} shows target inclusion at $H=2,4$ and target
exclusion at $H=8$. The adjusted $H=4$ minus $3.5(H=2)$ contrast contains
zero in both coordinates. These independent estimates support the predicted
local-step relation at $H=2,4$ and identify $H=8$ as more sensitive to the
fitted grid. Restricted-grid estimates and the separate joint-grid
$H=4$ discrepancy are documented in
Appendix~\ref{app:num-coefficient-results}.

\begin{table}[t]
\centering\small
\caption{Control-coefficient estimates on the independent Cartesian grid.
Each row compares the fitted intercept with $B_{20}^{\mathrm{ctrl}}(H)$.
Intervals are nominal 95\% run-seed intervals.}
\label{tab:supp-control}
\begin{tabular}{@{}ccrrl@{}}\toprule
$H$ & $j$ & Estimate & Target & Nominal interval \\ \midrule
2 & 1 & $-0.01599546$ & $-0.01600134$ & $[-0.01601159,\,-0.01597933]$ \\
2 & 2 & $-0.01999408$ & $-0.02000167$ & $[-0.02001422,\,-0.01997394]$ \\
4 & 1 & $-0.05595226$ & $-0.05600469$ & $[-0.05601109,\,-0.05589342]$ \\
4 & 2 & $-0.06993656$ & $-0.07000586$ & $[-0.07000999,\,-0.06986314]$ \\
8 & 1 & $-0.13962104$ & $-0.14001171$ & $[-0.13983973,\,-0.13940235]$ \\
8 & 2 & $-0.17449119$ & $-0.17501464$ & $[-0.17476345,\,-0.17421894]$ \\
\bottomrule
\end{tabular}
\end{table}

\subsubsection{Affine-shift equivalence}
\label{sec:num-affine}

\paragraph{Design and rationale.}
Proposition~\ref{prop:affine-shift} predicts exact equivalence after
client-specific gradient offsets are absorbed into the controls. We test
this identity by pairing affine-shift and homogeneous recursions with
translated initial controls and identical noise, across several shift
sizes, client counts, and local-step counts.

\paragraph{Results and interpretation.}
Across 24 configurations of 2,000 rounds, the maximum control-state
residual is $1.64313\times10^{-14}$. The numerical recursions therefore
reproduce the predicted state translation to floating-point accuracy.
This illustrates how correction of affine gradient offsets can coexist
with the same stochastic stationary-bias mechanism.

\subsubsection{Contraction cancellation and covariance-matched noise}
\label{sec:num-mechanism-controls}

\paragraph{Design and rationale.}
These tests isolate the specific role of $T[C]$. A nonquadratic
trigonometric objective has $T\ne0$, and the covariance family
$C_\rho=\operatorname{diag}(1+\rho,1-\rho)$ preserves total noise
variance. At $\rho=0$, the contraction vanishes while
$B_{11}=(-0.0125,0)^\top$ remains nonzero, isolating cancellation of the
client-independent coefficient. The positive and negative covariance
controls produce opposite-signed targets. At one fixed covariance,
bounded symmetric and asymmetric noise laws change third and fourth
moments while preserving the predicted coefficients. Unconstrained
Cartesian-grid fits estimate the individual and matched quantities;
the model and laws are given in Appendix~\ref{app:targeted-model}.

\paragraph{Results and interpretation.}
Table~\ref{tab:mechanism-controls} separates the matched null from both
nonzero covariance controls and recovers their predicted signs and
magnitudes. The noise-law contrast is small. In both coordinates, the
null and law-contrast intervals fall within the prespecified
coefficient-resolution margin $\pm0.004375$, equal to $20\%$ of the
nonzero matched target magnitude. Thus the matched increment responds
to covariance geometry and agrees across covariance-matched noise laws
at the reported resolution. Individual first-coordinate null and
law contrasts remain imprecise; one small paired second-coordinate
interval excludes zero while lying within the resolution margin.
Appendix~\ref{app:targeted-results} gives these complementary outcomes
and their sensitivity analyses.

\begin{table}[!htbp]
\centering\small
\setlength{\tabcolsep}{4pt}
\caption{Contraction and noise-law checks for the first coordinate of the
matched coefficient at $H=4$. The covariance controls predict positive,
zero, and negative coefficients; the last row compares noise laws with
identical covariance. Intervals are approximate simultaneous 95\% sampling
intervals for the prespecified 36-scalar family.}
\label{tab:mechanism-controls}
\begin{tabular}{@{}clrrl@{}}
\toprule
$\rho$ & Noise law & Target & Estimate & Simultaneous interval \\
\midrule
$-1/4$ & Symmetric & $0.021875$ & $0.02187858$ & $[0.02184128,\ 0.02191589]$ \\
$0$ & Symmetric & $0.000000$ & $-0.00002056$ & $[-0.00006150,\ 0.00002039]$ \\
$+1/4$ & Symmetric & $-0.021875$ & $-0.02185115$ & $[-0.02189536,\ -0.02180693]$ \\
$+1/4$ & Asymmetric & $-0.021875$ & $-0.02188179$ & $[-0.02193153,\ -0.02183205]$ \\
\midrule
\multicolumn{2}{@{}l}{Symmetric minus asymmetric ($\rho=1/4$)} & $0.000000$ & $0.00003064$ & $[-0.00003564,\ 0.00009692]$ \\
\bottomrule
\end{tabular}
\end{table}

\FloatBarrier

\subsection{Real-data implications of a fixed sampling budget}
\label{sec:num-real}

\paragraph{Design and rationale.}
We examine whether equal aggregate noise makes client allocations
interchangeable. All clients optimize the same a9a logistic-regression
objective, and the total minibatch budget $Nm$ is fixed. At the reference
minimizer $x_{\rm ref}$, the server-averaged noise covariance is
$\Sigma_{\rm ref}/(Nm)$, where $\Sigma_{\rm ref}$ is the single-example
gradient-noise covariance. The local covariance is instead
$\Sigma_{\rm ref}/m$, whose substitution into
Equation~\eqref{eq:ctrl-def} suggests a $\gamma^2/m$ dependence.
This motivates comparing allocations that have equal reference aggregate
noise and different local noise levels.

We measure finite-tail displacement along the preselected unit direction
$v\propto-A_{\rm ref}^{-1}T_{\rm ref}[\Sigma_{\rm ref}]$, with
$A_{\rm ref}=\nabla^2f(x_{\rm ref})$ and
$T_{\rm ref}=\nabla^3f(x_{\rm ref})$:
\begin{equation}
\label{eq:real-projection}
\widehat\delta_v^{\mathcal A}=v^\top(\overline x^{\mathcal A}-x_{\rm ref}),
\end{equation}
Here $\overline x^{\mathcal A}$ averages the per-run tail means across
independent runs. The statistic measures the signed mean displacement
along a direction fixed before the training runs.

\paragraph{Results and interpretation.}
Displacement increases as the same sample budget is distributed among
more clients with smaller local batches (Figure~\ref{fig:real-allocation}).
The $N=128$ minus $N=8$ contrasts are $0.003114$ for \Scaf{} and
$0.002072$ for \FedAvg{}, with adjusted intervals
$[0.002749,0.003478]$ and $[0.001713,0.002431]$. Raw variance estimates
remain near $0.00235$--$0.00236$. Equal reference aggregate covariance
therefore coexists with different finite-tail mean displacements:
the local allocation carries information that server-level covariance
alone does not capture.

\begin{figure}[!htbp]
\centering
\includegraphics[width=.58\textwidth]{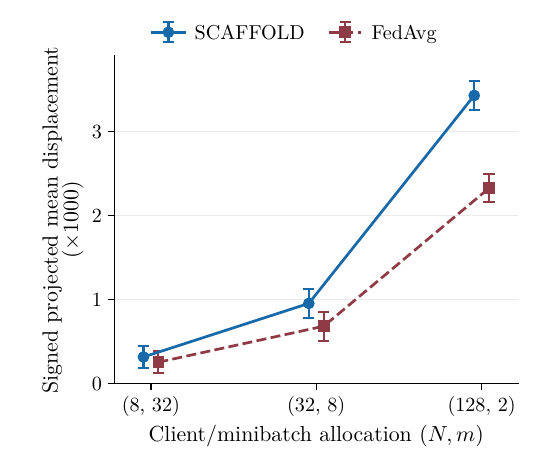}
\caption{Finite-tail mean displacement under a fixed sampling budget.
Each category gives a client-count/minibatch allocation $(N,m)$ with
$Nm=256$. The ordinate is the signed projection in
Equation~\eqref{eq:real-projection}, scaled by $10^3$.
Bars are nominal 95\% run-seed intervals; small offsets distinguish
the algorithms at the same allocation.}
\label{fig:real-allocation}
\end{figure}

The training-configuration comparison in Table~\ref{tab:supp-policy}
provides a complementary interpretation. Reducing $\gamma$ lowers both
displacement and raw variance; reducing $H$ lowers displacement with
much less variance change. These are comparisons under the run schedules
in Table~\ref{tab:num-run-schedules}. The disjoint IID and label-skew
results in Appendix~\ref{app:supp-real} additionally show how
heterogeneity-induced drift correction can dominate the algorithm
difference when the common-objective condition is relaxed.

\begin{table}[t]
\centering\small
\caption{Step-size and local-step comparisons for the shared objective at
$N=32,m=8$. All entries are multiplied by $10^3$.
$\widehat\delta_v$ is signed finite-tail mean displacement and
$\widehat V$ is raw-checkpoint parameter variance.
S: SCAFFOLD; F: FedAvg.}
\label{tab:supp-policy}
\begin{tabular}{@{}lrrrr@{}}\toprule
Setting & $\widehat\delta_v^{\rm S}$ & $\widehat\delta_v^{\rm F}$ & $\widehat V^{\rm S}$ & $\widehat V^{\rm F}$ \\ \midrule
$H=4,\ \gamma=\gamma_b$ & $0.9567$ & $0.6851$ & $2.3621$ & $2.3613$ \\
$H=4,\ \gamma=\gamma_b/2$ & $0.3547$ & $0.2796$ & $1.1723$ & $1.1722$ \\
$H=4,\ \gamma=\gamma_b/4$ & $0.0622$ & $0.0420$ & $0.5860$ & $0.5860$ \\
$H=2,\ \gamma=\gamma_b$ & $0.3402$ & $0.2474$ & $2.3601$ & $2.3599$ \\
\bottomrule
\end{tabular}
\end{table}

\FloatBarrier

\subsection{Model averaging changes the error decomposition}
\label{sec:num-averaging}

\paragraph{Design and rationale.}
Under stationarity, temporal averaging preserves the expected model
while potentially reducing variance. We therefore examine whether
mean displacement becomes a larger share of parameter error when
successive global models are averaged.
We use the same shared-objective runs and average successive global
models within each run over nested windows. Let $Y_{r,K}^{\mathcal A}$
be run $r$'s output averaged over its last $K$ equal-length batches.
Independent runs supply repeated outputs, while the windows share a
common endpoint. Table~\ref{tab:num-real-settings} gives the window
lengths and computational budget.

For a random output $Y_K^{\mathcal A}$ of this finite-window procedure,
\begin{equation}
\label{eq:e4-finite-window-mse}
\mathbb E\|Y_K^{\mathcal A}-x_{\rm ref}\|^2
=\|\mathbb E[Y_K^{\mathcal A}]-x_{\rm ref}\|^2
+\operatorname{tr}\operatorname{Cov}(Y_K^{\mathcal A}).
\end{equation}
We estimate squared mean displacement by the sampling-corrected cross-run
statistic $\widehat U_K$ and averaged-output variance by $\widehat V_K$.
Their sum is the estimated parameter MSE. Sampling correction allows
$\widehat U_K$ to be negative with finitely many runs; the plots retain
these values. Appendix~\ref{app:e4-averaging} gives the estimators.

\paragraph{Results and interpretation.}
In the primary allocation, mean displacement accounts for only about
$0.50\%$ and $0.23\%$ of raw-checkpoint parameter MSE for \Scaf{} and
\FedAvg{}, respectively. Figure~\ref{fig:e4-averaging} shows how this
composition changes after averaging. From $K=1$ to $64$, averaged-output
variance decreases about 57-fold. At $K=64$, the mean-displacement shares
reach $68.1\%$ and $49.5\%$, with simultaneous intervals
$[62.6\%,73.7\%]$ and $[41.1\%,57.9\%]$. Systematic displacement
therefore becomes a substantial part of the averaged-model error in
this experiment, even though individual checkpoints are variance-dominated.

\begin{figure}[!htbp]
\centering
\captionsetup[subfigure]{font=small}
\begin{subfigure}[t]{0.49\textwidth}
\centering
\includegraphics[width=\linewidth]{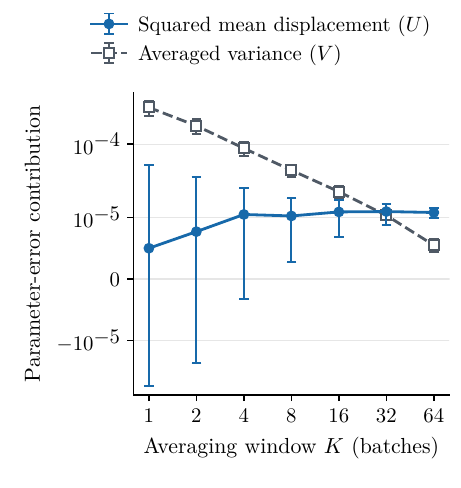}
\caption{SCAFFOLD}
\end{subfigure}\hfill
\begin{subfigure}[t]{0.49\textwidth}
\centering
\includegraphics[width=\linewidth]{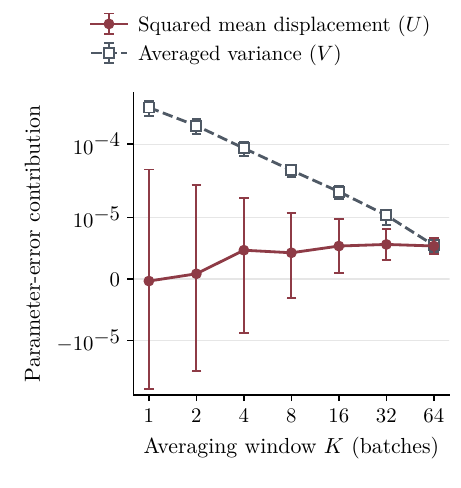}
\caption{FedAvg}
\end{subfigure}
\caption{Parameter-error decomposition after temporal model averaging.
The panels compare estimated squared mean displacement $\widehat U_K$
with averaged-output variance $\widehat V_K$ over windows of $K$ batches
in the primary shared-objective configuration. Each batch contains 1,280
rounds. Bars are approximate simultaneous 95\% primary-family intervals.
The symmetric-log axis is linear between $-10^{-5}$ and $10^{-5}$ and
retains negative sampling-corrected estimates and intervals crossing zero.}
\label{fig:e4-averaging}
\end{figure}

At the longest window, the \Scaf{}-minus-\FedAvg{} parameter-MSE
contrast is $6.39\times10^{-6}$, with interval
$[5.83,6.95]\times10^{-6}$, mainly associated with squared mean
displacement. The corresponding training-gap contrast is
$3.64\times10^{-7}$, while \Scaf{} has lower held-out log loss by
$3.27\times10^{-5}$. Thus parameter error and predictive loss respond
differently. The averaging result concerns finite-tail outputs;
Appendix~\ref{app:e4-averaging} reports the remaining temporal uncertainty.

\FloatBarrier

\subsection{Directional agreement beyond the additive-noise model}
\label{sec:num-directional}

\paragraph{Design and rationale.}
We examine whether the curvature--noise geometry helps explain the
algorithmic mean displacement under natural minibatch noise. The post-hoc
comparison substitutes the reference-objective quantities from
Section~\ref{sec:num-real} into the analytical expression:
\begin{equation}
\label{eq:e5-proxy}
p_{\gamma,H,m}
=-\frac{\gamma^2}{m}\,\kappa(H)\,A_{\rm ref}^{-1}T_{\rm ref}[\Sigma_{\rm ref}],
\qquad \kappa(H)=\frac{(H-1)(2H-1)}{12H}.
\end{equation}
The observed difference is
$\widehat\Delta_K=R^{-1}\sum_r(Y_{r,K}^{\Scaf}-Y_{r,K}^{\FedAvg})$
over $R=12$ paired runs. The proxy is computed directly from the objective
and reference covariance. Comparing all 124 coordinates tests its
geometric alignment with the observed mean difference. Cosine measures
directional agreement; projection ratio and relative residual measure
amplitude agreement and vector mismatch. Because both algorithms'
second-order terms share the same analytical direction, the amplitude and
local-step comparisons provide information beyond directional alignment.

\paragraph{Results and interpretation.}
At the designated longest window, the six nonzero-proxy configurations
have cosines $0.9965$--$0.9993$, projection ratios $0.716$--$0.938$, and
relative residuals $0.089$--$0.291$ (Table~\ref{tab:e5-bridge}). All six
proxy norms exceed the observed norms. At fixed $H,N,m$, reducing
$\gamma$ moves the projection ratio toward one and decreases the residual.
Alternative fixed directions have lower cosine estimates
(Appendix~\ref{app:e5-bridge}). The analytical geometry thus describes
the observed direction more accurately than its amplitude.

\begin{table}[!htbp]
\centering\small
\setlength{\tabcolsep}{5pt}
\caption{Post-hoc comparison of the analytical proxy $p=p_{\gamma,H,m}$
and the 124-dimensional finite-tail mean difference $\widehat\Delta_K$
at $K=64$. Cosine $=p^\top\widehat\Delta_K/(\|p\|\|\widehat\Delta_K\|)$;
projection ratio $=p^\top\widehat\Delta_K/\|p\|^2$;
relative residual $=\|\widehat\Delta_K-p\|/\|p\|$.
All six nonzero-proxy configurations are shown; the $H=1$ checks and
sampling intervals are in Appendix~\ref{app:e5-bridge}.}
\label{tab:e5-bridge}
\begin{tabular}{@{}rrrrrrr@{}}
\toprule
$H$ & $N$ & $m$ & $\gamma/\gamma_b$ & Cosine & \shortstack{Projection\\ratio} & \shortstack{Relative\\residual} \\
\midrule
2 & 32 & 8 & $1$ & $0.997719$ & $0.938327$ & $0.088505$ \\
4 & 8 & 32 & $1$ & $0.996500$ & $0.715730$ & $0.290542$ \\
4 & 32 & 8 & $1/4$ & $0.997252$ & $0.935781$ & $0.094644$ \\
4 & 32 & 8 & $1/2$ & $0.998972$ & $0.867499$ & $0.138225$ \\
4 & 32 & 8 & $1$ & $0.998994$ & $0.784711$ & $0.218151$ \\
4 & 128 & 2 & $1$ & $0.999254$ & $0.796703$ & $0.205615$ \\
\bottomrule
\end{tabular}
\end{table}

The local-step comparison further separates these two aspects of agreement:
at the base step size and $N=32,m=8$, the observed $H=4/H=2$ projected
ratio is $2.927$, compared with the analytical ratio $3.5$.
Together, the results give a directional explanation of the finite-tail
algorithm difference while identifying quantitative departures in magnitude
and local-step scaling. Appendix~\ref{app:e5-bridge} gives the sampling
intervals and window sensitivity.

\section{Conclusion and future directions}
\label{sec:conclusion}

Client averaging and control correction act on different parts of stationary mean bias. We made this distinction explicit for \Scaf{} and matched \FedAvg{}: the algorithms share a $\gamma/N$ term, while their client-independent $\gamma^2$ coefficients separate ordinary local-update and persistent-control contributions. Local second moments coupled with third-order curvature explain the additional control effect, and a uniform remainder identifies the coefficients along arbitrary joint small-step, large-client limits. An exact control translation extends the \Scaf{} result to affine objective shifts.

Controlled experiments examine these quantitative predictions. Real-data model
averaging illustrates that finite-tail mean displacement can be an important
component of parameter error, and a post-hoc full-vector comparison finds
close directional agreement with the derived curvature--noise structure despite
magnitude and local-step discrepancies. The analysis and experiments in this study explain an aspect of federated optimization that leading-order bias and
variance alone do not distinguish.

Several directions follow from this coefficient-level characterization:
\begin{itemize}
    \item \textbf{Nonlinear client heterogeneity:} Characterize how client-specific
    curvature and its interaction with local noise modify stationary bias
    beyond affine objective shifts.
    \item \textbf{Partial participation:} Determine how client subsampling and
    delayed control updates affect the stationary coefficients.
    \item \textbf{State-dependent noise:} Extend the analysis to gradient noise
    whose distribution varies with the local model.
    \item \textbf{Growing local-step counts:} Establish regimes in which $H$
    varies jointly with $N$ and $\gamma$.
    \item \textbf{Bias-aware control design:} Investigate whether alternative
    control updates can mitigate second-order bias while accounting for
    drift correction, variance, and computational cost.
\end{itemize}
\paragraph{Code availability.}
The companion repository for the implementations, detailed configurations,
numerical records, and reproduction instructions is
\url{https://github.com/Duck-YP/stationary-bias-study}.


\clearpage
\begingroup
\raggedright
\bibliographystyle{plainnat}
\bibliography{references}
\endgroup

\clearpage
\appendix
\section{Common moment estimates and the SCAFFOLD proof}
\label{app:scaf-proof}

The proof separates the moment estimates shared by both algorithms from the
additional contribution of SCAFFOLD controls. We first bound local fluctuations,
then obtain the sharp global raw second moment and a stationary mean-balance identity
with a uniform remainder. The SCAFFOLD control covariance completes
Theorem~\ref{thm:scaf-main}; Appendix~\ref{app:fa-proof} uses the same
estimates for FedAvg and the matched attribution.

\subsection{Notation and stationary inputs}
\label{app:proof-notation}
\label{app:common-inputs}

The three scales used to track errors uniformly in $N$ are
\begin{equation}
\label{eq:scales}
 s:=s_{\gamma,N}=\frac{\gamma}{N}+\gamma^2,\qquad
 r:=r_{\gamma,N}=\frac{\gamma^2}{N}+\gamma^3=\gamma s,\qquad
 v:=v_{\gamma,N}=\frac{\gamma^2}{N^2}+\gamma^3.
\end{equation}
Constants denoted by $K$ may change between occurrences and depend on fixed
$H,d$ and the objective/noise regularity constants, but not on $N$ or $\gamma$.
All $O(\cdot)$ terms in these proofs have this same uniform meaning.
We use Euclidean norms for vectors and tensors, induced norms for linear
maps, and $\|\cdot\|_F$ for the Frobenius matrix norm. In particular,
$\|P^h\|$ is an operator norm and operators on matrices use the
Frobenius-induced norm. We use repeatedly, for
$\gamma\le1$ and $N\ge2$,
\begin{equation}
\label{eq:common-scale-comparisons}
 v\le r,\qquad \gamma^2\sqrt{s}\le r,\qquad
 \gamma^2(s+N^{-1})\le2r,\qquad \gamma^4\le\gamma v.
\end{equation}
For example, $\sqrt{s}\le\gamma+N^{-1}$ follows by squaring both sides.
For vectors $u,v$, set $(T[u,v])_i=\sum_{j,k}T_{ijk}u_jv_k$.
In cubic identities, $\Sym$ is the unnormalized sum of three placements:
\[
 [\Sym(u\otimes M)]_{ijk}=u_iM_{jk}+u_jM_{ik}+u_kM_{ij},\qquad
 [\Sym(u^{\otimes2}\otimes v)]_{ijk}
 =u_iu_jv_k+u_iv_ju_k+v_iu_ju_k,
\]
with $M$ symmetric. This convention extends linearly under expectation.

We use stationary moment bounds from \citet[Corollary~B.4 and
Lemma~B.5]{mangold2025scaffold} and \citet[Proposition~3 and
Lemmas~5--6]{mangold2025fedavg}. These results apply here because the
noise laws are fixed across rounds, $F_\eps(y)=f(y)+\langle\eps,y\rangle$
has the same Hessian as $f$, and bounded additive noise has the required
moments without state-dependent growth. Identical client objectives remove
the heterogeneity term. A common sufficient step-size threshold is
\begin{equation}
\label{eq:imported-step-threshold}
 0<\gamma\le\gamma_{\mathrm{imp}}
 :=\min\left\{1,\frac1{48L},\frac1{24H(L+\mu)}\right\}.
\end{equation}
It also implies the FedAvg higher-moment restrictions $\gamma\le1/(9L)$
and $\gamma\mu H\le1$. The additional SCAFFOLD variance-growth restrictions
are vacuous because their parameter $\beta$ is zero here. We decrease
$\gamma_0\le\gamma_{\mathrm{imp}}$, independently of $N$, for subsequent
contraction and absorption inequalities.

For each algorithm under its own stationary law, the imported estimates give
\begin{equation}
\label{eq:A-coarse-even}
 \E\|x\|^{2p}+\frac1N\sum_c\E\|\theta_{c,h}\|^{2p}
 =O(\gamma^p),\qquad p=1,2,3,\quad 0\le h\le H.
\end{equation}
For SCAFFOLD, additionally, we have
\begin{equation}
\label{eq:A-control-sixth}
 \sup_c\E\|\xi_c\|^6\le K,\qquad \sum_c\xi_c=0\quad\text{pathwise}.
\end{equation}
In the common lemmas below, write $\xi_c=0$ for FedAvg. This notation
unifies only the within-round local recursions; it does not impose the
SCAFFOLD control-refresh rule on FedAvg. In either case, controls are fixed
during a round, and fresh noise is independent of the complete round-start state.

We set $\delta_{c,h}=\theta_{c,h}-x$, $\bar\delta_h=N^{-1}\sum_c\delta_{c,h}$,
$\bar\eps_j=N^{-1}\sum_c\eps_{c,j}$, and
\begin{equation}
\label{eq:A-qP}
 \begin{aligned}
 q(y)&=\nabla f(y)-Ay,\qquad P=I-\gamma A,\qquad b=\E[x],\\
 m_h&=\frac1N\sum_c\E[\theta_{c,h}],\qquad
 \nu_h=\frac1N\sum_c\E[q(\theta_{c,h})].
 \end{aligned}
\end{equation}
Also, we write
\begin{equation}
\label{eq:A-MU}
 M_x=\E[xx^\top],\qquad M_h=\frac1N\sum_c\E[\theta_{c,h}\theta_{c,h}^\top],
 \qquad U_h=M_h-M_x,\qquad Q_\xi=\frac1N\sum_c\E[\xi_c\xi_c^\top].
\end{equation}
These are raw second moments, evaluated under each algorithm's own
stationary law. In particular, $U_h$ measures the additional second moment
accumulated during the $h$ local steps.

\subsection{Common coarse and local second-moment estimates}
\label{app:common-coarse}

\begin{lemma}[Coarse scales and local second moments]
\label{lem:common-coarse}
For either algorithm and $0\le h\le H$,
\begin{align}
 \frac1N\sum_c\E\|\delta_{c,h}\|^{2p}&=O(\gamma^{2p}),\quad p=1,2,3,
 \label{eq:A-delta-moments}\\
 \E\|x\|^2+\|b\|&=O(s),\qquad m_h-b=O(r),
 \label{eq:A-mean-bootstrap}\\
 \E\|\bar\delta_h\|^2&=O\bigl(\gamma^2(s+N^{-1})\bigr),
 \label{eq:A-deltabar}\\
 U_h&=\gamma^2(hC+h^2Q_\xi)+O(r).
 \label{eq:A-U-pre}
\end{align}
\end{lemma}

\begin{proof}
Unrolling the fixed-$H$ local recursion and using Equations~\eqref{eq:A-coarse-even}
and~\eqref{eq:A-control-sixth} gives Equation~\eqref{eq:A-delta-moments}
(the control bound is automatic for FedAvg).
At the server, the client control corrections cancel by the zero-sum
condition. The resulting round increment and stationary balance are
\begin{align}
 x^+-x&=-\gamma\sum_{h=0}^{H-1}\frac1N\sum_c\nabla f(\theta_{c,h})
         -\gamma\sum_{j=1}^H\bar\eps_j,
 \label{eq:A-global-increment}\\
 0&=\sum_{h=0}^{H-1}\frac1N\sum_c\E[\nabla f(\theta_{c,h})].
 \label{eq:A-balance}
\end{align}
\emph{Coarse global moment and mean.}
Write $x^+=\Phi(x)+\zeta+E_{\mathrm{traj}}$, where
$\Phi(x)=x-\gamma H\nabla f(x)$, $\zeta=-\gamma\sum_{j=1}^H\bar\eps_j$,
and $E_{\mathrm{traj}}=-\gamma\sum_hN^{-1}\sum_c
\{\nabla f(\theta_{c,h})-\nabla f(x)\}$.
Strong convexity gives $\|\Phi(x)\|^2\le(1-c\gamma)\|x\|^2$ for some $c>0$,
while $\E\|\zeta\|^2=O(\gamma^2/N)$ and
$\E\|E_{\mathrm{traj}}\|^2=O(\gamma^4)$.
First, we use $\E\langle\Phi(x),\zeta\rangle=0$, then apply Young's inequality
only to $E_{\mathrm{traj}}$. For fixed sufficiently small $a>0$, we derive
\[
 \E\|x\|^2\le(1+a\gamma)
 \{(1-c\gamma)\E\|x\|^2+K\gamma^2/N\}
 +K_a\gamma^{-1}\E\|E_{\mathrm{traj}}\|^2.
\]
Absorption gives $\E\|x\|^2=O(s)$ and hence
$N^{-1}\sum_c\E\|\theta_{c,h}\|^2=O(s)$.
Since $\|q(y)\|\le K\|y\|^2$, we have $\nu_h=O(s)$.
The averaged recursion $m_{h+1}=Pm_h-\gamma\nu_h$, with $m_0=m_H=b$, gives
\[
 (I-P^H)b=-\gamma\sum_{j=0}^{H-1}P^{H-1-j}\nu_j,\qquad
 m_h-b=(P^h-I)b-\gamma\sum_{j=0}^{h-1}P^{h-1-j}\nu_j.
\]
Using $\|(I-P^H)^{-1}\|=O(\gamma^{-1})$ and $\|P^h-I\|=O(\gamma)$
proves Equation~\eqref{eq:A-mean-bootstrap}. For the client-averaged displacement, the controls again cancel. Its
gradient part has second moment
$O(\gamma^2s)$ and the averaged-noise part has second moment $O(\gamma^2/N)$.
Their sum gives Equation~\eqref{eq:A-deltabar} without assuming the two parts
are independent.

\emph{Local second moments.}
To isolate the effect of fresh noise and persistent controls, write
$\theta_{c,h}=x+D_{c,h}+R_{c,h}$ with
\begin{equation}
\label{eq:A-DR}
 D_{c,h}=-\gamma\left(h\xi_c+\sum_{j=1}^h\eps_{c,j}\right),\qquad
 R_{c,h}=-\gamma\sum_{\ell=0}^{h-1}\nabla f(\theta_{c,\ell}).
\end{equation}
Zero-sum cancellation and freshness imply
\begin{equation}
\label{eq:A-DD}
 \frac1N\sum_c\E[xD_{c,h}^\top]=0,\qquad
 \frac1N\sum_c\E[D_{c,h}D_{c,h}^\top]=\gamma^2(hC+h^2Q_\xi).
\end{equation}
Only the client-averaged first identity is needed; $x$ and individual controls
need not be uncorrelated. Furthermore,
$N^{-1}\sum_c\E\|D_{c,h}\|^2=O(\gamma^2)$ and
$N^{-1}\sum_c\E\|R_{c,h}\|^2=O(\gamma^2s)$.
Cauchy--Schwarz bounds the averaged $xR^\top$, $DR^\top$, and $RR^\top$
terms by $O(\gamma s)$, $O(\gamma^2\sqrt{s})$, and $O(\gamma^2s)$,
respectively. All are $O(r)$ by Equation~\eqref{eq:common-scale-comparisons},
which proves Equation~\eqref{eq:A-U-pre}.
\end{proof}

\subsection{Common higher-moment bootstrap}
\label{app:common-higher}

The signed third moment is needed to control the cubic term in the
stationary Taylor expansion. We obtain it after a fourth-moment estimate.

\begin{lemma}[Fourth and signed third moments]
\label{lem:common-higher}
For either algorithm,
\begin{equation}
\label{eq:A-higher-moments}
 \E\|x\|^4=O(v),\qquad \|\E[x^{\otimes3}]\|=O(v).
\end{equation}
\end{lemma}

\begin{proof}
Unrolling the averaged local recursion gives
\begin{equation}
\label{eq:A-round-exact}
 x^+=P^Hx+\eta+\Delta,\qquad
 \eta=-\gamma\sum_{j=1}^HP^{H-j}\bar\eps_j,\quad
 \Delta=-\gamma\sum_{h=0}^{H-1}P^{H-1-h}\bar q_h,
 \quad \bar q_h=\frac1N\sum_c q(\theta_{c,h}).
\end{equation}
The centered fresh term $\eta$ is independent of the round-start state.
Expansion of sums of independent centered vector blocks gives
\begin{equation}
\label{eq:A-eta-moments}
 \E\|\eta\|^2=O(\gamma^2/N),\qquad
 \E\|\eta\|^4=O(\gamma^4/N^2),\qquad
 \|\E[\eta^{\otimes3}]\|=O(\gamma^3/N^2).
\end{equation}
No symmetry of the noise law is used. Global Lipschitzness of $\nabla f$
and bounded $\nabla^3f$ imply $\|q(y)\|\le K\min\{\|y\|,\|y\|^2\}$.
Thus $\|q(y)\|^2\le K\|y\|^4$ and $\|q(y)\|^4\le K\|y\|^6$;
Jensen's inequality and Equation~\eqref{eq:A-coarse-even} yield
\begin{equation}
\label{eq:A-Delta-moments}
 \E\|\Delta\|^2=O(\gamma^4),\qquad \E\|\Delta\|^4=O(\gamma^7).
\end{equation}
In particular, an eighth-moment assumption is unnecessary.

\emph{Fourth moment.}
Let $Y=P^Hx+\eta$ and $M_4=\E\|x\|^4$.
Independence and centeredness of $\eta$ give
\[
 \E\|Y\|^4\le\|P^H\|^4M_4+
 K\left\{\E\|x\|^2\E\|\eta\|^2
 +\|b\|\sqrt{\E\|\eta\|^2\E\|\eta\|^4}+\E\|\eta\|^4\right\}
 \le\|P^H\|^4M_4+K\gamma v.
\]
Use $\|y+z\|^4\le(1+a\gamma)\|y\|^4+K_a\gamma^{-3}\|z\|^4$
and $\|P^H\|^4\le1-c\gamma$. For fixed sufficiently small $a>0$,
stationarity implies
\[
 M_4\le(1+a\gamma)\{(1-c\gamma)M_4+K\gamma v\}
      +K_a\gamma^{-3}\E\|\Delta\|^4.
\]
The last term is $O(\gamma^4)\subset O(\gamma v)$; absorption proves
$M_4=O(v)$ and also $\E\|Y\|^4=O(v)$.

\emph{Signed third moment.}
For $T_3=\E[x^{\otimes3}]$, the linear part satisfies
\[
 \E[Y^{\otimes3}]=(P^H)^{\otimes3}T_3
 +\Sym\{P^Hb\otimes\E[\eta\eta^\top]\}+\E[\eta^{\otimes3}].
\]
The two forcing terms are $O(\gamma v)$ by the coarse mean bound and
Equation~\eqref{eq:A-eta-moments}.
Set $a_h=N^{-1}\sum_c\|\theta_{c,h}\|^2$. By the fourth-moment bound
just proved and Equation~\eqref{eq:A-delta-moments},
\[
 \E[a_h^2]\le\frac1N\sum_c\E\|\theta_{c,h}\|^4
 \le K(v+\gamma^4)=O(v),\qquad
 \E[\|Y\|^2a_h]\le\sqrt{\E\|Y\|^4\E[a_h^2]}=O(v).
\]
Since $\|\Delta\|\le K\gamma\sum_ha_h$, it follows that
$\|\E[Y^{\otimes2}\otimes\Delta]\|=O(\gamma v)$.
The other terms in the cubic expansion obey
\[
 \E[\|Y\|\|\Delta\|^2]\le
 \sqrt{\E\|Y\|^2\E\|\Delta\|^4}=O(\gamma^4),\qquad
 \E\|\Delta\|^3\le(\E\|\Delta\|^4)^{3/4}=O(\gamma^{21/4}),
\]
both $O(\gamma v)$. Hence,
$(I-(P^H)^{\otimes3})T_3=O(\gamma v)$; its inverse has norm
$O(\gamma^{-1})$, proving $T_3=O(v)$. The sharp global second moment can
therefore be established using these higher-moment bounds.
\end{proof}

\subsection{Common sharp global second moment}
\label{app:common-covariance}

We now compare the nonlinear stationary chain with its linear counterpart.
The leading global second moment comes from server-averaged noise; the
higher-moment bounds control the nonlinear correction.

\begin{lemma}[Sharp stationary raw second moment]
\label{lem:A-sharp-second-moment}
For either algorithm, the nonlinear forcing in the stationary second-moment
equation is $O(\gamma r)$, and
\begin{equation}
\label{eq:A-sharp-Mx}
 M_x=\frac{\gamma}{N}S_C+O(r).
\end{equation}
\end{lemma}

\begin{proof}
\emph{Linear chain.}
For the linear chain $x_{\rm lin}^+=P^Hx_{\rm lin}+\eta$,
$\E[\eta\eta^\top]=(\gamma^2/N)\sum_{k=0}^{H-1}P^kC(P^k)^\top$.
Its stationary second-moment equation is the $H$-fold iteration of
$M_{\rm lin}=PM_{\rm lin}P^\top+\gamma^2C/N$.
Both equations have unique solutions because $P$ is contractive.
In an eigenbasis of $A$, with eigenvalues $a_i$,
\begin{equation}
\label{eq:A-Mlin-coord}
 (M_{\rm lin})_{ij}=\frac{\gamma}{N}
 \frac{C_{ij}}{a_i+a_j-\gamma a_ia_j},\qquad
 (S_C)_{ij}=\frac{C_{ij}}{a_i+a_j}.
\end{equation}
Consequently $M_{\rm lin}=(\gamma/N)S_C+O(\gamma^2/N)$ without
assuming $AC=CA$.

\emph{Nonlinear state correction.}
To control $\E[x\Delta^\top]$, expand
$\bar q_h=q(x)+Dq(x)\bar\delta_h+\bar\rho_h$,
where $\|Dq(x)\|\le K\|x\|$ and
$\|\bar\rho_h\|\le K N^{-1}\sum_c\|\delta_{c,h}\|^2$.
The expansion $q(x)=\frac12T[x,x]+O(\|x\|^3)$ and the preceding lemmas give
\begin{align*}
 \|\E[xq(x)^\top]\|_F&\le K\{\|\E[x^{\otimes3}]\|+\E\|x\|^4\}=O(r),\\
 \E[\|x\|^2\|\bar\delta_h\|]&\le
 \sqrt{\E\|x\|^4\E\|\bar\delta_h\|^2}
 =O\bigl(\sqrt{v\gamma^2(s+N^{-1})}\bigr)=O(r),\\
 \E[\|x\|\|\bar\rho_h\|]&=O(\gamma^2\sqrt{s})=O(r).
\end{align*}
The outer factor $\gamma$ in $\Delta$ therefore yields
\begin{equation}
\label{eq:A-xDelta}
 \|\E[x\Delta^\top]\|_F=O(\gamma r).
\end{equation}

\emph{Fresh-noise/nonlinear correlation.}
To control $\E[\eta\Delta^\top]$, hold the complete round-start state fixed.
Index the $NH$ fresh blocks by $i=(c,j)$ and write
$\eta=\sum_iB_i\eps_i$, with $\|B_i\|\le K\gamma/N$.
Let $\Delta^{(i,0)}$ replace block $\eps_i$ by zero.
Only client $c$ changes before aggregation. Lipschitz propagation gives a
path change of at most $K\gamma\|\eps_i\|$ and hence
\[
 \|\Delta-\Delta^{(i,0)}\|\le K\frac{\gamma^2}{N}\|\eps_i\|,\qquad
 \E[\eps_i(\Delta^{(i,0)})^\top\mid\text{round-start state}]=0.
\]
Each term $\E[B_i\eps_i(\Delta-\Delta^{(i,0)})^\top]$ is
$O(\gamma^3/N^2)$; summing the $NH$ terms gives
\begin{equation}
\label{eq:A-etaDelta}
 \|\E[\eta\Delta^\top]\|_F=O(\gamma^3/N).
\end{equation}
Replacing one noise block at a time handles the dependence between
$\eta$ and $\Delta$. Finally,
$\|\E[\Delta\Delta^\top]\|_F=O(\gamma^4)$ by
Equation~\eqref{eq:A-Delta-moments}.

\emph{Stationary covariance equation.}
For $\mathcal S_\gamma(X)=X-P^HX(P^H)^\top$, stationarity gives
\begin{align}
\label{eq:A-second-moment-difference}
 \mathcal S_\gamma(M_x-M_{\rm lin})
 ={}&P^H\E[x\Delta^\top]+\E[\Delta x^\top](P^H)^\top\notag\\
 &+\E[\eta\Delta^\top]+\E[\Delta\eta^\top]+\E[\Delta\Delta^\top].
\end{align}
The right-hand side is $O(\gamma r)$, since both $\gamma^3/N$ and
$\gamma^4$ are bounded by $\gamma r$. Moreover,
$\|\mathcal S_\gamma^{-1}\|\le(1-\|P^H\|^2)^{-1}=O(\gamma^{-1})$.
Thus $M_x-M_{\rm lin}=O(r)$, proving Equation~\eqref{eq:A-sharp-Mx}.
\end{proof}

\subsection{Common signed Taylor closure}
\label{app:common-taylor}

The global raw second moment determines one part of the mean bias. Stationary
gradient balance also involves the second moments along the local
trajectories. The following lemma combines both through a Taylor expansion.

\begin{lemma}[Local Taylor tail and stationary mean balance]
\label{lem:common-taylor}
For either algorithm,
\begin{align}
 \left\|\frac1N\sum_c\E[\theta_{c,h}^{\otimes3}]\right\|&=O(r),\qquad
 \frac1N\sum_c\E\|\theta_{c,h}\|^4=O(v),
 \label{eq:A-local-moments}\\
 b&=-\frac12A^{-1}T[M_x]
     -\frac1{2H}A^{-1}T\!\left[\sum_{h=0}^{H-1}U_h\right]+O(r).
 \label{eq:A-master}
\end{align}
\end{lemma}

\begin{proof}
Expanding $\theta_{c,h}=x+\delta_{c,h}$ gives
\begin{align*}
 \frac1N\sum_c\E[\theta_{c,h}^{\otimes3}]
 ={}&\E[x^{\otimes3}]+\Sym\{\E[x^{\otimes2}\otimes\bar\delta_h]\}\\
 &+\Sym\!\left\{\frac1N\sum_c\E[x\otimes\delta_{c,h}^{\otimes2}]\right\}
 +\frac1N\sum_c\E[\delta_{c,h}^{\otimes3}].
\end{align*}
The four terms have norms bounded respectively by
\[
 O(v),\qquad O\bigl(\sqrt{v\gamma^2(s+N^{-1})}\bigr),\qquad
 O(\gamma^2\sqrt{s}),\qquad O(\gamma^3),
\]
using Cauchy--Schwarz and Lemmas~\ref{lem:common-coarse}
and~\ref{lem:common-higher}. All are $O(r)$.
Also $N^{-1}\sum_c\E\|\theta_{c,h}\|^4
\le K\{\E\|x\|^4+N^{-1}\sum_c\E\|\delta_{c,h}\|^4\}
=O(v+\gamma^4)=O(v)$.
The signed cubic and absolute fourth-moment bounds now control the
higher-order terms in the gradient expansion. With
$D_4=\nabla^4f(0)$, Taylor's theorem gives
\begin{equation}
\label{eq:A-Taylor}
 \nabla f(u)=Au+\frac12T[u,u]+\frac16D_4[u,u,u]+\rho_5(u),\qquad
 \|\rho_5(u)\|\le K\|u\|^4.
\end{equation}
Here $D_4[u,u,u]$ contracts the fourth derivative against three vectors.
Equation~\eqref{eq:A-local-moments} controls the signed cubic term and
absolute fourth-order remainder, giving
\begin{equation}
\label{eq:A-Taylor-tail}
 \left\|\sum_{h=0}^{H-1}\frac1N\sum_c
 \E\!\left[\nabla f(\theta_{c,h})-A\theta_{c,h}
             -\frac12T[\theta_{c,h},\theta_{c,h}]\right]\right\|=O(r).
\end{equation}
Substitute into Equation~\eqref{eq:A-balance}, use
$\sum_hm_h=Hb+O(r)$ from Lemma~\ref{lem:common-coarse}, and write
$M_h=M_x+U_h$. This gives
$0=HAb+\frac12T[\sum_hM_h]+O(r)$ and proves
Equation~\eqref{eq:A-master}.
\end{proof}

\subsection{SCAFFOLD control moment and coefficient assembly}
\label{app:scaf-assembly}

The common lemmas reduce the algorithm-specific calculation to the
stationary covariance $Q_\xi$ of the persistent controls. Define
\begin{equation}
\label{eq:A-GE}
 G_c=\frac1H\sum_{h=0}^{H-1}
 \left\{\nabla f(\theta_{c,h})-\frac1N\sum_k\nabla f(\theta_{k,h})\right\},
 \qquad E_c=\frac1H\sum_{j=1}^H(\eps_{c,j}-\bar\eps_j).
\end{equation}
The control update in Equation~\eqref{eq:scaf-server} gives
$\xi_c^+=-G_c-E_c$ exactly. Independence across clients and local steps yields
$\E[E_cE_c^\top]=(1-1/N)C/H$. Smoothness and
Equation~\eqref{eq:A-delta-moments} give
$N^{-1}\sum_c\E\|G_c\|^2=O(\gamma^2)$, while
$N^{-1}\sum_c\E\|E_c\|^2=O(1)$.
The $GE^\top$ and $EG^\top$ terms are therefore $O(\gamma)$ by
Cauchy--Schwarz, not by independence. Control stationarity implies
\begin{equation}
\label{eq:A-Q}
 Q_\xi=Q_{\gamma,N,H}=\frac1H\left(1-\frac1N\right)C+O(\gamma).
\end{equation}
Inserting this covariance into the local second-moment expansion adds
the $h^2$ contribution specific to persistent controls:
\begin{equation}
\label{eq:A-U}
 U_h=\gamma^2\left(h+\frac{h^2}{H}\right)C+O(r),
\end{equation}
because the omitted pieces are $O(\gamma^2/N+\gamma^3)$.
Finally,
\begin{equation}
\label{eq:A-U-sum}
 \sum_{h=0}^{H-1}U_h
 =\gamma^2\frac{(H-1)(5H-1)}6C+O(r).
\end{equation}
Combining Equations~\eqref{eq:A-sharp-Mx}, \eqref{eq:A-master},
and~\eqref{eq:A-U-sum} yields
\[
 b^{\Scaf}_{\gamma,N,H}
 =-\frac{\gamma}{2N}A^{-1}T[S_C]
 -\gamma^2\frac{(H-1)(5H-1)}{12H}A^{-1}T[C]+O(r).
\]
This proves Theorem~\ref{thm:scaf-main}, including the stated uniform remainder.

\subsection{Proof of the affine-control translation}
\label{app:affine-proof}
\begin{proof}[Proof of Proposition~\ref{prop:affine-shift}]
The gradient is $\nabla f_c(y)=\nabla f(y)+a_c$. With the translation in Equation~\eqref{eq:affine-objectives}, the heterogeneous local step becomes
\[
\theta_{c,h+1}=\theta_{c,h}-\gamma\{\nabla f(\theta_{c,h})+\widetilde\xi_c+\eps_{c,h+1}\}.
\]
The zero-sum constraint is preserved because $\sum_ca_c=0$, and the refreshed state satisfies
\[
\widetilde\xi_c^+=\xi_c^++a_c
=\widetilde\xi_c+\frac{\theta_{c,H}-x^+}{\gamma H}.
\]
Thus the translated controls give exactly the homogeneous round map,
with the global model $x$ unchanged. Induction yields pathwise equality
for corresponding initial states driven by the same noise blocks.

The translation is invertible and carries invariant laws between the
two chains. Uniqueness of the homogeneous invariant law establishes
equality of their global stationary marginals. Since
$N^{-1}\sum_cf_c=f$, both biases are measured from the same minimizer,
and Theorem~\ref{thm:scaf-main} gives the expansion and uniform remainder.
\end{proof}

\section{FedAvg proof and matched attribution}
\label{app:fa-proof}

For FedAvg, persistent controls are absent, so the common estimates of
Appendix~\ref{app:scaf-proof} apply with $\xi_c=0$. We verify their
inputs and then compute the local second moment that determines the FedAvg
coefficient. All expectations are under the FedAvg invariant law, with
$s,r,v$ defined in Equation~\eqref{eq:scales}.

\subsection{Application of the common estimates}

The stationary moment estimates of \citet[Proposition~3 and
Lemmas~5--6]{mangold2025fedavg} give
Equation~\eqref{eq:A-coarse-even} under the step-size threshold in
Appendix~\ref{app:common-inputs}. Equation~\eqref{eq:fa-local} is the
common local recursion with $\xi_c=0$, and the independent fresh noise
satisfies the assumptions of the common lemmas. The two round identities
used there become
\begin{align}
 x^+-x&=-\gamma\sum_{h=0}^{H-1}\frac1N\sum_c\nabla f(\theta_{c,h})
         -\gamma\sum_{j=1}^H\bar\eps_j,
 \label{eq:B-global-increment}\\
 x^+&=P^Hx-\gamma\sum_{j=1}^HP^{H-j}\bar\eps_j
             -\gamma\sum_{h=0}^{H-1}P^{H-1-h}\bar q_h.
 \label{eq:B-round-exact}
\end{align}
These identities establish the same stationary balance and nonlinear
round decomposition used in Appendix~\ref{app:scaf-proof}. Applying
Lemmas~\ref{lem:common-coarse} and~\ref{lem:common-higher} first gives
the required local and higher-moment bounds. Lemma~\ref{lem:A-sharp-second-moment}
then identifies the global raw second moment in Equation~\eqref{eq:A-sharp-Mx},
and Lemma~\ref{lem:common-taylor} supplies the stationary mean-balance
identity in Equation~\eqref{eq:A-master}. Only the local second-moment
increment remains to be evaluated.

\subsection{Local coefficient and completion of the proof}

Without persistent controls, $Q_\xi=0$ and the stochastic part of the
local displacement contains only fresh noise. Equations~\eqref{eq:A-DR}
and~\eqref{eq:A-DD} therefore reduce to
\[
 D_{c,h}=-\gamma\sum_{j=1}^h\eps_{c,j},\qquad
 \frac1N\sum_c\E[xD_{c,h}^\top]=0,\qquad
 \frac1N\sum_c\E[D_{c,h}D_{c,h}^\top]=\gamma^2hC.
\]
The already proved remainder estimate in Lemma~\ref{lem:common-coarse} yields
\begin{equation}
\label{eq:B-U}
 U_h=\gamma^2hC+O(r),\qquad
 \sum_{h=0}^{H-1}U_h=\gamma^2\frac{H(H-1)}2C+O(r).
\end{equation}
Substituting this and $M_x=(\gamma/N)S_C+O(r)$ into
Equation~\eqref{eq:A-master} gives
\begin{equation}
\label{eq:B-final-expansion}
 b^{\FedAvg}_{\gamma,N,H}
 =-\frac{\gamma}{2N}A^{-1}T[S_C]
 -\gamma^2\frac{H-1}{4}A^{-1}T[C]
 +O_{H,d}\!\left(\frac{\gamma^2}{N}+\gamma^3\right),
\end{equation}
proving Theorem~\ref{thm:fa-main}.
Subtracting the two explicit second-order coefficients gives
\[
 B_{20}^{\Scaf}(H)-B_{20}^{\FedAvg}(H)
 =-\frac{(H-1)(2H-1)}{12H}A^{-1}T[C]
 =B_{20}^{\mathrm{ctrl}}(H).
\]
The common $\gamma/N$ term cancels in the matched comparison, while
the difference of the two uniform remainders remains
$O_{H,d}(\gamma^2/N+\gamma^3)$. This completes
Corollary~\ref{cor:attribution}.

\section{Numerical protocols, results, and diagnostics}
\label{app:num-details}

The appendix groups supporting material by purpose: controlled models and
inference (Section~\ref{app:num-design}), coefficient uncertainty and
grid sensitivity (Section~\ref{app:num-coefficient-results}), estimator
and structural checks (Section~\ref{app:num-consistency}), and real-data
analyses (Section~\ref{app:num-real-analyses}). Parameter settings are
collected in tables; supplementary results use the same design-and-interpretation
structure as the main experiments.

\subsection{Controlled experimental design and inference}
\label{app:num-design}

\subsubsection{Controlled instances and estimation}
\label{app:num-instance}

The primary nonlinear family provides known curvature, noise covariance,
and coefficient targets:
\begin{align}
 f(x)&=\tfrac12x^\top Ax+\alpha\psi(u^\top x),\qquad
 \psi(0)=0,\quad\psi'(z)=\log\cosh z,
 \label{eq:num-objective}\\
 \eps&=L_Cz,\qquad L_CL_C^\top=C,
 \label{eq:num-noise}
\end{align}
The coordinates of $z$ are independent Rademacher variables, with fresh
noise blocks across clients, steps, rounds, and simulation units.
The two-dimensional instance uses
\begin{equation}
\label{eq:num-d2}
 A=\begin{pmatrix}1.2&0.2\\0.2&1.0\end{pmatrix},\qquad
 C=\begin{pmatrix}1.0&0.2\\0.2&0.7\end{pmatrix},\qquad
 u=\frac{(1,1)^\top}{\sqrt2},\quad\alpha=0.25.
\end{equation}
It has $x^\star=0$, $AC\ne CA$, and $T[M]=\alpha u(u^\top Mu)$.
The leading coefficient is
\[
B_{11}=(-0.0244594282,-0.0305742852)^\top.
\]
Table~\ref{tab:num-controlled-settings} collects the experimental grids.

\begin{table}[tbp]
\centering\small
\setlength{\tabcolsep}{4pt}
\caption{Controlled experimental settings. Runs are independent simulation
units per configuration; the matched studies use algorithm pairs.
The smaller-step path supplements $H=4$ on the primary joint grid.
Structural-check settings are given in Table~\ref{tab:num-structural-settings}.}
\label{tab:num-controlled-settings}
\begin{tabularx}{\linewidth}{@{}>{\raggedright\arraybackslash}p{0.25\linewidth}
>{\raggedright\arraybackslash}p{0.14\linewidth}>{\raggedright\arraybackslash}Xc@{}}
\toprule
Comparison & $d$; $H$ & Client counts and step sizes & Runs \\
\midrule
Primary joint grid & $2$; $2,4$ &
\begin{minipage}[t]{\linewidth}
\vspace{0pt}
\begin{equation}
\label{eq:num-primary-grid}
\begin{gathered}
N\in\{12,16,24,32,48\},\\ \gamma=1/N.
\end{gathered}
\end{equation}
\end{minipage} & 8 \\
Smaller-step joint grid & $2$; $4$ &
\begin{minipage}[t]{\linewidth}
\vspace{0pt}
\begin{equation}
\label{eq:num-small-grid}
\begin{gathered}
N\in\{64,96,128,192\},\\ \gamma=1/N.
\end{gathered}
\end{equation}
\end{minipage} & 8 \\
Fixed-step client count & $2$; $2$ &
$N\in\{8,16,32,64,128\}$; $\gamma=1/24$ & 8 \\
Five-dimensional & $5$; $2,4$ &
$N\in\{32,48,64,96,128\}$; $\gamma=1/N$ & 8 \\
Leading coefficient & $2$; $1,2,4,8$ &
$N=32$; $\gamma\in\{1/64,1/128,1/256,1/512\}$ & 12 \\
Cartesian control coefficient & $2$; $2,4,8$ &
$N\in\{32,64,128\}$;
$\gamma\in\{1/128,1/256,1/512\}$ & 12 \\
Contraction/noise law & $2$; $4$ &
$N\in\{64,128\}$;
$\gamma\in\{1/512,1/1024,1/2048\}$ & 12 \\
\bottomrule
\end{tabularx}
\end{table}

\paragraph{Mean-balance estimator.}
Let $P_\gamma=I-\gamma A$, $q(y)=\nabla f(y)-Ay$, and
$\bar q_h=N^{-1}\sum_cq(\theta_{c,h})$, using pre-update local states.
A round contributes
\begin{equation}
\label{eq:num-balance-estimator}
 \widehat b_{\rm bal}
 =-\gamma(I-P_\gamma^H)^{-1}\sum_{h=0}^{H-1}P_\gamma^{H-1-h}\bar q_h.
\end{equation}
Averaging this observable over post-burn-in rounds gives $\widehat b$.
The linear part of the gradient is handled analytically, so the observable
focuses on the nonlinear contribution to the stationary mean. The direct
global mean supplies a secondary check in
Appendix~\ref{app:direct-balance-diagnostic}.

For $\delta_{\gamma,N}=\gamma+1/N$, the joint-grid extrapolation is
\begin{equation}
\label{eq:num-linear-fit}
 \widehat{\mathcal C}_j=\beta_{0,j}+\beta_{1,j}\delta_{\gamma,N}
\end{equation}
with the analogous fit for $\widehat D_j$. The intercepts are compared
with the coefficients defined by the joint small-step, large-client limit.
Intervals describe sampling uncertainty conditional on the fitted grid;
transient effects and deterministic extrapolation error are assessed
through the sensitivity checks below.

\subsubsection{Run schedules and sampling uncertainty}
\label{app:num-protocol-details}

Table~\ref{tab:num-run-schedules} records the simulation schedules.
The time scales use curvature-based heuristics. Within-chain batches are
used for accumulation and window comparisons; independent runs are the
sampling units for uncertainty estimation.

\begin{table}[tbp]
\centering\small
\setlength{\tabcolsep}{4pt}
\caption{Run schedules. Burn-in and batch length are measured in
communication rounds. The real-data regularization parameter $\lambda$
is given in Table~\ref{tab:num-real-settings}.}
\label{tab:num-run-schedules}
\begin{tabularx}{\linewidth}{@{}>{\raggedright\arraybackslash}p{0.23\linewidth}>{\raggedright\arraybackslash}Xccc@{}}
\toprule
Study & Heuristic time scale & Burn-in & Batch length & Batches \\
\midrule
Joint grid, fixed step, and $d=5$ &
$m_{\rm mix}=\lceil(\mu_{\rm mix}H\gamma)^{-1}\rceil$;
$\mu_{\rm mix}=0.85$ ($d=2$), $0.8037572789$ ($d=5$)
& $100m_{\rm mix}$ & $20m_{\rm mix}$ & 100 \\
Cartesian/local step &
$\tau_0=\lceil(0.85H\gamma)^{-1}\rceil$
& $50\tau_0$ & $10\tau_0$ & 64 \\
Contraction/noise law &
$\tau=\lceil(0.6H\gamma)^{-1}\rceil$
& $50\tau$ & $10\tau$ & 64 \\
Real-data training &
$\tau_0=\lceil(\lambda H\gamma)^{-1}\rceil$
& $50\tau_0$ & $10\tau_0$ & 64 \\
\bottomrule
\end{tabularx}
\end{table}

For the joint-grid studies, the precision criterion is
\[
 h_j\le\max\{0.0025,\,0.10|B_{20,j}^{\rm target}(H)|\},
\]
where $h_j$ is the Bonferroni-adjusted Student-$t$ half-width.
Every reported run met this criterion at the first 100-batch check;
the planned schedule allowed increments of 50 up to 400 batches.
Run-validity checks use finite states, a coordinatewise stationarity
diagnostic at $z=3$, at least seven valid units out of eight, and a
$10^{-10}$ control-sum tolerance.

Table~\ref{tab:num-inference} summarizes the interval constructions.
For the joint-grid fits, a resampled unit is a complete run-level grid
vector. Matched outputs stay together, preserving their covariance in
coefficient differences. The trigonometric study instead resamples run
vectors independently within each configuration, reflecting its independent
streams across grid cells.

\begin{table}[tbp]
\centering\small
\setlength{\tabcolsep}{4pt}
\caption{Sampling-uncertainty procedures. Simultaneous families are
specified separately for the listed estimands. All resampling preserves
matched algorithm outputs and their coordinates.}
\label{tab:num-inference}
\begin{tabularx}{\linewidth}{@{}>{\raggedright\arraybackslash}p{0.24\linewidth}
>{\raggedright\arraybackslash}p{0.37\linewidth}>{\raggedright\arraybackslash}X@{}}
\toprule
Study & Interval construction & Coordinate levels / simultaneous family \\
\midrule
Joint grid and fixed step, $d=2$ &
Pointwise Student-$t$ (7 degrees of freedom); 2,000 seed-block bootstrap draws for intercepts &
Coordinate level $0.975$; bootstrap quantiles $(0.0125,0.9875)$ \\
Five-dimensional &
Pointwise Student-$t$ (7 degrees of freedom); 5,000 bootstrap draws &
Coordinate level $0.99$; quantiles $(0.005,0.995)$ \\
Cartesian/local step &
Student-$t$ (11 degrees of freedom); Bonferroni adjustment; 2,000 bootstrap draws for sensitivity &
16 leading coordinates; 6 control coordinates; 12 $H$ contrasts; 2 coordinates for $H=4-3.5(H=2)$ \\
Contraction/noise law &
4,000 within-cell bootstrap draws; centered standardized maximum; marginal and Bonferroni-$t$ sensitivity &
36 scalars: 24 coefficient coordinates, 6 law contrasts, 6 covariance contrasts \\
Real-data allocation &
Nominal jackknife-$t$ intervals; adjusted contrast intervals &
105 cellwise projections; separate 36 allocation/configuration contrasts \\
Model averaging &
19,999 paired-seed bootstrap draws; centered standardized maximum; delete-one-seed Bonferroni-$t$ sensitivity &
Separate primary (196), secondary (392), and diagnostic (1,176) families \\
Directional comparison &
19,999 paired-seed resamples and jackknife sensitivity &
Separate 95\% scalar (1,068) and vector (117) families \\
\bottomrule
\end{tabularx}
\end{table}

\subsubsection{Cartesian-grid and local-step estimation}
\label{app:num-cartesian-protocol}

The independent study comprises 34 controlled configurations.
Each run contributes a grid vector, and OLS is applied runwise before
averaging fitted coefficients. The designs are $(1,\gamma)$ for leading
coefficients and $(1,\gamma,1/N)$ for matched control coefficients.
A longer trajectory at $(H,N,\gamma)=(4,64,1/256)$ supplies shifted and
extended windows as within-run sensitivity analyses. The primary estimates
use the schedule in Table~\ref{tab:num-run-schedules}.

\subsubsection{Trigonometric contraction and noise-law design}
\label{app:targeted-model}

The trigonometric objective makes covariance-dependent cancellation
explicit. With $\alpha=1/20$, define
\begin{equation}
\label{eq:targeted-objective}
\begin{aligned}
f(x)&=\tfrac12(x_1^2+2x_2^2)+\alpha g(x),\\
g(x)&=-4(\sin x_1-x_1)
 +\sin(x_1+x_2)-(x_1+x_2)
 +\sin(x_1-x_2)-(x_1-x_2).
\end{aligned}
\end{equation}
Its perturbation has Hessian norm at most $8\alpha=0.4$, so
$0.6I\preceq\nabla^2 f(x)\preceq2.4I$ globally. Derivatives of orders
three to five are bounded, $x^\star=0$, and $A=\operatorname{diag}(1,2)$.
For symmetric $M$,
\[
T[M]=(2\alpha(M_{11}-M_{22}),-4\alpha M_{12})^\top.
\]
Thus the covariance family in Section~\ref{sec:num-mechanism-controls}
gives $S_{C_\rho}=\operatorname{diag}((1+\rho)/2,(1-\rho)/4)$ and
$T[C_\rho]=(\rho/5,0)^\top$. At $H=4$,
\begin{equation}
\label{eq:targeted-coefficients}
\begin{aligned}
 B_{11}(\rho)&=\left(-\frac{1+3\rho}{80},0\right)^\top,
 &B_{20}^{\FedAvg}(4;\rho)&=\left(-\frac{3\rho}{20},0\right)^\top,\\
 B_{20}^{\Scaf}(4;\rho)&=\left(-\frac{19\rho}{80},0\right)^\top,
 &B_{20}^{\mathrm{ctrl}}(4;\rho)&=\left(-\frac{7\rho}{80},0\right)^\top.
\end{aligned}
\end{equation}

Noise is $\varepsilon=C_\rho^{1/2}z$, with independent standardized
coordinates. Table~\ref{tab:num-noise-laws} gives the two bounded laws.
They share covariance $C_\rho$ and the common bound
$\|\varepsilon\|\le\sqrt5$. Fresh draws follow these laws directly.
The resolved coefficients depend on covariance, while the remainder
allows law differences at order $\gamma^2/N+\gamma^3$.

\begin{table}[tbp]
\centering\small
\caption{Covariance-matched scalar noise laws. Both have mean zero and
variance one. Symmetric noise is used at $\rho=-1/4,0,1/4$; asymmetric
noise is also tested at $\rho=1/4$.}
\label{tab:num-noise-laws}
\begin{tabular}{@{}lcccc@{}}
\toprule
Law & Support & Probabilities & $\mathbb E[z_i^3]$ & $\mathbb E[z_i^4]$ \\
\midrule
Symmetric & $(-1,1)$ & $(1/2,1/2)$ & $0$ & $1$ \\
Asymmetric & $(\sqrt2,-1/\sqrt2)$ & $(1/3,2/3)$ & $1/\sqrt2$ & $3/2$ \\
\bottomrule
\end{tabular}
\end{table}

The full-grid fit uses unweighted OLS with intercept and regressors
$(\gamma,1/N)$. Both normalized individual and matched observables are
estimated freely. Noise is shared within each algorithm pair; grid cells,
noise laws, pilot runs, and production runs have independent streams.
The only restricted-grid sensitivity removes the largest step size.

For null and law-difference comparisons, the prespecified resolution
margin is $20\%$ of the nonzero first-coordinate target magnitude:
\[
\begin{array}{c|ccc}
 & \Scaf{} & \FedAvg{} & \text{Matched difference}\\ \hline
\text{Margin} & 0.011875 & 0.0075 & 0.004375
\end{array}
\]
Resolution at a margin requires the entire interval to lie inside its
symmetric band. These fits use two client counts, so the coefficients of
the descriptive grid fit and the asymptotic remainder rate have separate
interpretations.

The proof-based step-size bound provides a reference for the tested grids.
Several joint-grid steps exceed the imported threshold in
Equation~\eqref{eq:imported-step-threshold}; all trigonometric steps are
below $\gamma_{\rm imp}=1/288$. The final sufficient $\gamma_0$ also
incorporates the contraction and absorption conditions of the proofs.

\subsection{Coefficient estimates and finite-grid diagnostics}
\label{app:num-coefficient-results}

\subsubsection{Joint-grid coefficient intervals}
\label{app:num-coefficients}

Table~\ref{tab:num-core} supplies the simultaneous intervals for the point
estimates and targets in Table~\ref{tab:num-attribution-summary}.
The separate individual and paired studies have different sampling
precision; their $H=4$ target-inclusion outcomes are reported separately
in Sections~\ref{sec:num-coefficients} and~\ref{sec:num-attribution}.

\begin{table}[tbp]
\centering\small
\setlength{\tabcolsep}{4pt}
\caption{Simultaneous 95\% seed-block-bootstrap intervals for the
joint-grid coefficients in Table~\ref{tab:num-attribution-summary}.
S: SCAFFOLD; F: FedAvg; D: matched difference.
P and S identify the primary and smaller-step grids.}
\label{tab:num-core}
\label{tab:E-fa}
\label{tab:E-control}
\begin{tabular}{@{}cccll@{}}
\toprule
Study & $H$ & Grid & Coordinate 1 interval & Coordinate 2 interval \\
\midrule
S & 2 & P & $[-0.0482374,-0.0471965]$ & $[-0.0602886,-0.0589883]$ \\
S & 4 & P & $[-0.1508906,-0.1497518]$ & $[-0.1882351,-0.1868165]$ \\
S & 4 & S & $[-0.1522435,-0.1507781]$ & $[-0.1902854,-0.1884530]$ \\
\addlinespace
F & 2 & P & $[-0.0326464908,-0.0313780188]$ & $[-0.0408082738,-0.0392220804]$ \\
F & 4 & P & $[-0.0965774402,-0.0953321996]$ & $[-0.1205630991,-0.1190074303]$ \\
F & 4 & S & $[-0.0964743358,-0.0953790293]$ & $[-0.1205859653,-0.1192168647]$ \\
\addlinespace
D & 2 & P & $[-0.0160963688,-0.0159914783]$ & $[-0.0201109827,-0.0199811131]$ \\
D & 4 & P & $[-0.0544427680,-0.0541825657]$ & $[-0.0678381597,-0.0675216076]$ \\
D & 4 & S & $[-0.0559502332,-0.0559007437]$ & $[-0.0699256032,-0.0698639209]$ \\
\bottomrule
\end{tabular}
\end{table}

\subsubsection{Matched \texorpdfstring{$H=4$}{H=4} finite-grid diagnostics}
\label{app:num-h4-diagnostics}

\paragraph{Design and rationale.}
We examine how the remaining matched error changes with grid resolution.
For each point of the combined joint grids, set
$e_{\Delta,j}=\widehat D_j-B_{20,j}^{\mathrm{ctrl}}(4)$ and compare its
magnitude with $\delta=\gamma+1/N$. A quadratic fit to the same nine
points checks sensitivity to the linear extrapolation.

\paragraph{Results and interpretation.}
Table~\ref{tab:E-h4diag} and Figure~\ref{fig:E-control-h4} show an
approximately linear residual trend in $\delta$. The quadratic fit gives
intercepts $-0.056008582$ and $-0.070000016$, with target errors
$3.90\times10^{-6}$ and $5.84\times10^{-6}$. Thus the extrapolated
intercept is sensitive to the finite-grid fit, and the grid-point residuals
are compatible with the normalized $O(\gamma+1/N)$ scale. Since both
paths tie $\gamma$ to $1/N$, the step-size and client-count contributions
remain combined. The source of the remaining discrepancy, including
finite-run effects, is unresolved; Table~\ref{tab:num-core} retains the
primary linear-fit intervals.

\begin{table}[tbp]
\centering\small
\caption{Matched $H=4$ grid-point residuals against
$\delta=\gamma+1/N$. The last two columns summarize the observed
residual scale across the combined grid.}
\label{tab:E-h4diag}
\begin{tabular}{@{}ccccc@{}}
\toprule
$j$ & $|e_\Delta|$ at $\delta_{\min}$ & $|e_\Delta|$ at $\delta_{\max}$ &
$|e_\Delta|/\delta$ range & Log--log slope \\
\midrule
1 & $0.0015750$ & $0.0202050$ & $[0.121,0.151]$ & $0.927$\\
2 & $0.0020785$ & $0.0263460$ & $[0.158,0.200]$ & $0.923$\\
\bottomrule
\end{tabular}
\end{table}

\begin{figure}[!htbp]
\centering
\captionsetup[subfigure]{font=small}
\begin{subfigure}[t]{0.49\textwidth}
\centering
\includegraphics[width=\linewidth]{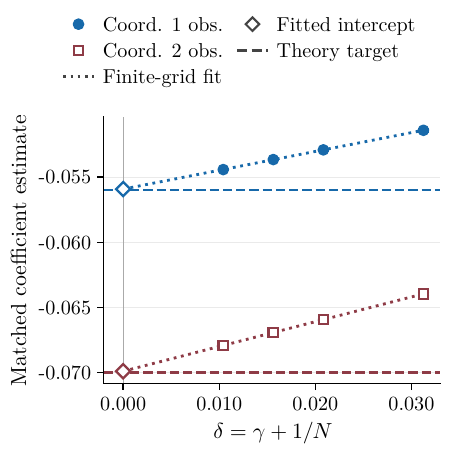}
\caption{Smaller-step path}
\end{subfigure}\hfill
\begin{subfigure}[t]{0.49\textwidth}
\centering
\includegraphics[width=\linewidth]{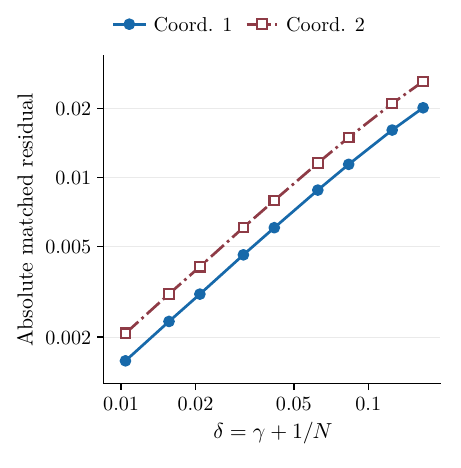}
\caption{Absolute residuals}
\end{subfigure}
\caption{Matched $H=4$ finite-grid sensitivity.
(a) Smaller-step estimates (circles and squares), finite-grid linear fits
(dotted), and analytical targets (dashed). Open diamonds at $\delta=0$
mark the fitted intercepts.
(b) Absolute residuals versus $\delta=\gamma+1/N$ across all nine points
on logarithmic axes; lines connect the observed residuals.
The scaled residuals and log--log slopes are given in
Table~\ref{tab:E-h4diag}. Pointwise bars in (a) use simultaneous
Student-$t$ intervals; intercept bars use seed-block-bootstrap intervals.}
\label{fig:E-control-h4}
\label{fig:E-h4diag}
\end{figure}

\subsubsection{Fixed-step and five-dimensional results}
\label{app:num-supporting-results}

\paragraph{Design and rationale.}
These studies separately examine client count and dimension using the
objective family in Equation~\eqref{eq:num-objective}. The fixed-step
study uses the grid in Table~\ref{tab:num-controlled-settings}.
The five-dimensional instance has positive-definite $A,C$ with
$\|AC-CA\|_F=0.0876171216$.

\paragraph{Results and interpretation.}
Table~\ref{tab:E-plateau} adds the fixed-step intercept intervals and
offsets to Figure~\ref{fig:E-p2-plateau}. Table~\ref{tab:E-d5} supplies
all five coordinates for each local-step count. These details support
the client-count and multidimensional interpretations in
Sections~\ref{sec:num-supporting} and~\ref{sec:num-multidimensional}.

\begin{table}[tbp]
\centering\small
\caption{Sampling intervals and offsets for the fixed-step extrapolated
intercepts in Figure~\ref{fig:E-p2-plateau}.}
\label{tab:E-plateau}
\begin{tabular}{@{}clr@{}}
\toprule
$j$ & 95\% simultaneous CI & Intercept $-$ theory \\
\midrule
1 & $[-0.0489079,-0.0483504]$ & $-0.0006102$ \\
2 & $[-0.0607622,-0.0600656]$ & $-0.0003902$ \\
\bottomrule
\end{tabular}
\end{table}

\begin{table}[tbp]
\centering\small
\setlength{\tabcolsep}{4pt}
\caption{Five-dimensional \Scaf{} coefficient targets, fitted
intercepts, and nominal simultaneous 95\% bootstrap intervals.}
\label{tab:E-d5}
\begin{tabular}{@{}ccrrl@{}}
\toprule
$H$ & $j$ & $\widehat\beta_{0,j}$ & Theory & 95\% simultaneous CI \\
\midrule
2 & 1 & $-0.0140655061$ & $-0.0140249988$ & $[-0.0142709800,-0.0138295249]$ \\
2 & 2 & $-0.0457104417$ & $-0.0455796216$ & $[-0.0463785328,-0.0449434842]$ \\
2 & 3 & $ 0.0322203570$ & $ 0.0321284384$ & $[ 0.0316797243, 0.0326913914]$ \\
2 & 4 & $-0.0340954937$ & $-0.0339978616$ & $[-0.0345938212,-0.0335234201]$ \\
2 & 5 & $ 0.0160014114$ & $ 0.0159557997$ & $[ 0.0157329196, 0.0162353381]$ \\
\midrule
4 & 1 & $-0.0444005615$ & $-0.0444124962$ & $[-0.0446094313,-0.0441636371]$ \\
4 & 2 & $-0.1441927099$ & $-0.1443354682$ & $[-0.1448724780,-0.1434224100]$ \\
4 & 3 & $ 0.1016052163$ & $ 0.1017400551$ & $[ 0.1010621370, 0.1020846930]$ \\
4 & 4 & $-0.1075543334$ & $-0.1076598952$ & $[-0.1080613670,-0.1069797680]$ \\
4 & 5 & $ 0.0504598160$ & $ 0.0505266990$ & $[ 0.0501901098, 0.0506979332]$ \\
\bottomrule
\end{tabular}
\end{table}

\subsubsection{Leading-coefficient and restricted-grid sensitivity}
\label{app:supp-controlled}

\paragraph{Design and rationale.}
The independent fits estimate the common leading coefficient directly.
Restricted-grid comparisons omit the largest step size or, for the
matched control fits, the smallest client count. This assesses sensitivity
to points farther from the joint asymptotic limit.

\paragraph{Results and interpretation.}
Table~\ref{tab:supp-leading} gives the individual leading-intercept
estimates, complementing the matched results in
Table~\ref{tab:supp-control}. The $H=8$ offsets persist as a sensitivity
issue even when restricted fits move toward the targets. Primary estimates
remain the basis of the comparisons in Section~\ref{sec:num-supp-theory}.

\begin{table}[tbp]
\centering\small
\caption{Independent leading-coefficient fits with nominal 95\% intervals.
S: SCAFFOLD; F: FedAvg. The targets are
$B_{11,1}=-0.02445943$ and $B_{11,2}=-0.03057429$ for every $H$;
$H=1$ checks the boundary recursion.}
\label{tab:supp-leading}
\begin{tabular}{@{}ccrlrl@{}}\toprule
$H$ & Alg. & Estimate$_1$ & Interval$_1$ & Estimate$_2$ & Interval$_2$ \\ \midrule
1 & S & $-0.024531$ & $[-0.024960,-0.024102]$ & $-0.030664$ & $[-0.031200,-0.030127]$ \\
1 & F & $-0.024531$ & $[-0.024960,-0.024102]$ & $-0.030664$ & $[-0.031200,-0.030127]$ \\
2 & S & $-0.024432$ & $[-0.025068,-0.023795]$ & $-0.030554$ & $[-0.031349,-0.029758]$ \\
2 & F & $-0.024413$ & $[-0.025051,-0.023774]$ & $-0.030525$ & $[-0.031323,-0.029728]$ \\
4 & S & $-0.025169$ & $[-0.025953,-0.024385]$ & $-0.031537$ & $[-0.032517,-0.030557]$ \\
4 & F & $-0.024799$ & $[-0.025574,-0.024024]$ & $-0.031044$ & $[-0.032013,-0.030075]$ \\
8 & S & $-0.028252$ & $[-0.029046,-0.027459]$ & $-0.035610$ & $[-0.036602,-0.034618]$ \\
8 & F & $-0.025913$ & $[-0.026712,-0.025114]$ & $-0.032560$ & $[-0.033559,-0.031562]$ \\
\bottomrule\end{tabular}\end{table}

\subsubsection{Contraction and noise-law resolution}
\label{app:targeted-results}

\paragraph{Design and rationale.}
The supplementary comparisons assess how the matched result relates to
individual coefficients, the second coordinate, and restricted grids.
All use the model and resolution margins in
Appendix~\ref{app:targeted-model}. The signed-covariance contrast
subtracts the $\rho=-1/4$ symmetric condition from $\rho=1/4$.

\paragraph{Results and interpretation.}
Table~\ref{tab:mechanism-complement} supplements the matched
first-coordinate results in Table~\ref{tab:mechanism-controls}.
The paired signed-covariance estimates and simultaneous intervals are
\[
\begin{aligned}
 j=1:&\quad -0.04372973,\qquad [-0.04378778,-0.04367168],\\
 j=2:&\quad -0.00002676,\qquad [-0.00005541,0.00000188],
\end{aligned}
\]
against target $(-0.04375,0)^\top$. Individual first-coordinate null
and noise-law intervals are wider than their resolution bands, so the
precision gain is concentrated in the paired increment.

\begin{table}[!htbp]
\centering\small
\setlength{\tabcolsep}{3pt}
\caption{Additional contraction and noise-law results: individual coefficients
and contrasts, and the paired second coordinate. Estimates and interval
endpoints are multiplied by $10^3$.
S: \Scaf{}; F: \FedAvg{}; D: matched difference.
Targets follow Equation~\eqref{eq:targeted-coefficients}.
The law contrast is symmetric minus asymmetric at $\rho=1/4$;
the covariance contrast is $\rho=1/4$ minus $-1/4$ with symmetric noise.
Intervals use the primary family in Table~\ref{tab:mechanism-controls}.}
\label{tab:mechanism-complement}
\begin{tabular}{@{}lcrlrl@{}}
\toprule
Condition / contrast & Alg. & Estimate$_1$ & Interval$_1$ & Estimate$_2$ & Interval$_2$ \\
\midrule
$-1/4$, sym. & S & $56.0908$ & $[42.4573,\ 69.7243]$ & $2.6156$ & $[-4.5559,\ 9.7871]$ \\
$-1/4$, sym. & F & $34.2122$ & $[20.5757,\ 47.8487]$ & $2.6118$ & $[-4.5609,\ 9.7844]$ \\
\addlinespace
$0$, sym. & S & $4.6861$ & $[-18.0942,\ 27.4663]$ & $2.4973$ & $[-6.0117,\ 11.0064]$ \\
$0$, sym. & F & $4.7066$ & $[-18.0689,\ 27.4822]$ & $2.5118$ & $[-5.9994,\ 11.0230]$ \\
\addlinespace
$+1/4$, sym. & S & $-46.8325$ & $[-67.9207,\ -25.7443]$ & $2.2185$ & $[-4.0669,\ 8.5038]$ \\
$+1/4$, sym. & F & $-24.9813$ & $[-46.0714,\ -3.8913]$ & $2.2414$ & $[-4.0463,\ 8.5291]$ \\
\addlinespace
$+1/4$, asym. & S & $-58.1461$ & $[-81.4153,\ -34.8769]$ & $2.7156$ & $[-4.1149,\ 9.5460]$ \\
$+1/4$, asym. & F & $-36.2643$ & $[-59.5301,\ -12.9986]$ & $2.7160$ & $[-4.1137,\ 9.5457]$ \\
\addlinespace
Law contrast & S & $11.3136$ & $[-20.1352,\ 42.7624]$ & $-0.4971$ & $[-9.8863,\ 8.8921]$ \\
Law contrast & F & $11.2830$ & $[-20.1627,\ 42.7287]$ & $-0.4747$ & $[-9.8649,\ 8.9156]$ \\
\addlinespace
Covariance contrast & S & $-102.9233$ & $[-127.9666,\ -77.8799]$ & $-0.3971$ & $[-9.8315,\ 9.0372]$ \\
Covariance contrast & F & $-59.1935$ & $[-84.2420,\ -34.1451]$ & $-0.3704$ & $[-9.8073,\ 9.0665]$ \\
\addlinespace
\midrule
\multicolumn{6}{@{}l}{Additional paired second-coordinate results (zero target)} \\
$-1/4$, sym. & D & \multicolumn{2}{c}{---} & $0.0038$ & $[-0.0175,\ 0.0251]$ \\
$0$, sym. & D & \multicolumn{2}{c}{---} & $-0.0145$ & $[-0.0349,\ 0.0059]$ \\
$+1/4$, sym. & D & \multicolumn{2}{c}{---} & $-0.0229$ & $[-0.0423,\ -0.0035]$ \\
$+1/4$, asym. & D & \multicolumn{2}{c}{---} & $-0.0005$ & $[-0.0213,\ 0.0204]$ \\
Law contrast & D & \multicolumn{2}{c}{---} & $-0.0224$ & $[-0.0511,\ 0.0062]$ \\
\bottomrule
\end{tabular}
\end{table}

\paragraph{Restricted-grid sensitivity.}
Table~\ref{tab:num-restricted-mechanism} records the three restricted-grid
comparisons whose simultaneous intervals narrowly exclude their targets.
The two individual coefficient estimates lose sign resolution.
All three corresponding Bonferroni-$t$ intervals include the targets,
showing sensitivity to the interval construction as well as the grid.

\begin{table}[tbp]
\centering\small
\setlength{\tabcolsep}{4pt}
\caption{First-coordinate sensitivity after removing the largest step
in the contraction study. The first two rows use $\rho=1/4$ and symmetric
noise; the last row is the signed-covariance contrast.}
\label{tab:num-restricted-mechanism}
\begin{tabular}{@{}lrrl@{}}
\toprule
Quantity & Target & Estimate & Simultaneous interval \\
\midrule
\Scaf{} coefficient & $-0.059375$ & $-0.01675171$ & $[-0.05907084,0.02556741]$ \\
\FedAvg{} coefficient & $-0.0375$ & $0.00510137$ & $[-0.03721159,0.04741432]$ \\
\Scaf{} covariance contrast & $-0.11875$ & $-0.06840058$ & $[-0.11872928,-0.01807187]$ \\
\bottomrule
\end{tabular}
\end{table}

For the $\rho=1/4$ symmetric paired second coordinate, the primary
interval excludes zero but is contained in the resolution band. Its
Bonferroni-$t$ and restricted-grid intervals include zero. Reflection
symmetry of this objective and symmetric-noise recursion fixes the ideal
second-coordinate mean at zero from zero initial states. The small
empirical offset therefore probes finite-simulation and interval accuracy,
while the resolution-band comparison measures the size of that offset.

First/last-half balance comparisons give eight nominal 95\% exclusions
among 144 scalar diagnostics, all in individual algorithms; paired
intervals contain zero. Direct means have substantially larger sampling
variation than balance estimates. The diagnostics assess complementary
sources of uncertainty: window comparisons assess temporal variation,
finite-window identities assess arithmetic and alignment, and restricted
fits assess grid sensitivity. Stationary approximation remains unresolved.

\subsection{Estimator and structural consistency checks}
\label{app:num-consistency}

\subsubsection{Direct--balance estimator checks}
\label{app:direct-balance-diagnostic}

\paragraph{Design and rationale.}
The direct global mean provides a cross-check of the primary nonlinear
balance observable. The unnormalized diagnostic is
\[
 |\widehat b_{{\rm direct},j}-\widehat b_{{\rm bal},j}|
 \le3\operatorname{SE}_{\rm comb},\qquad
 \operatorname{SE}_{\rm comb}
 =\sqrt{\operatorname{SE}_{{\rm direct},j}^2+
         \operatorname{SE}_{{\rm bal},j}^2}.
\]
Using the unnormalized means keeps the comparison on the parameter scale;
second-order normalization multiplies estimation errors by $\gamma^{-2}$.

\paragraph{Results and interpretation.}
The primary two-dimensional \FedAvg{} study satisfies the diagnostic in
156 of 160 chain-coordinate checks. Table~\ref{tab:E-direct-balance}
gives the four exceptions; their discrepancy-to-threshold ratios are
$1.137$, $1.164$, $1.108$, and $1.060$. All checks satisfy the criterion
in the two-dimensional \Scaf{} studies, fixed-step study,
five-dimensional \Scaf{} study, and smaller-step \FedAvg{} study.
These counts concern the individual algorithms. In the independent
controlled study, four of 68 algorithm--configuration raw-versus-balance
intervals exclude zero. The primary estimates throughout use the
balance observable.

\begin{table}[tbp]
\centering\small
\setlength{\tabcolsep}{3pt}
\caption{The four direct--balance discrepancies in the primary
\FedAvg{} study. The diagnostic uses the unnormalized difference
$|\Delta\widehat b|$; the normalized statistics are also shown.}
\label{tab:E-direct-balance}
\begin{tabular}{@{}cccrrrr@{}}
\toprule
$H$ & $N$ & $j$ & $\widehat C_{\rm bal}$ & $\widehat C_{\rm dir}$ &
$|\Delta\widehat b|$ & $3\operatorname{SE}_{\rm comb}$ \\
\midrule
2 & 16 & 2 &
$-0.0408082$ & $ 0.924598$ & $0.00377112$ & $0.00331567$ \\
2 & 48 & 1 &
$-0.0323449$ & $ 3.070686$ & $0.00134680$ & $0.00115744$ \\
4 & 16 & 2 &
$-0.1080240$ & $ 0.841309$ & $0.00370833$ & $0.00334677$ \\
4 & 48 & 1 &
$-0.0933450$ & $ 2.615233$ & $0.00117560$ & $0.00110953$ \\
\bottomrule
\end{tabular}
\end{table}

Reanalysis of the existing $H=4$ Cartesian-grid runs gives paired
raw-minus-balance fitted-intercept differences of
$-2.39\times10^{-7}$ and $-2.92\times10^{-7}$. The paired intercepts
therefore change little under estimator substitution. Finite-window
identities agree where endpoints are available. Missing endpoints and
individual noise sums limit this check to those windows; the joint-grid
$H=4$ discrepancy remains as reported in
Appendix~\ref{app:num-h4-diagnostics}.

\subsubsection{Orthogonal coordinate equivariance}
\label{app:equivariance-check}
\label{app:equivariance}

\paragraph{Design and rationale.}
A coordinate change should transform the complete recursion and its
coefficient vectors consistently. For $y=Qx$ with $QQ^\top=I$, set
$A_y=QAQ^\top$, $C_y=QCQ^\top$, and $S_{C,y}=QS_CQ^\top$.
The Lyapunov equation becomes $A_yS_{C,y}+S_{C,y}A_y=C_y$, and
\begin{equation}
\label{eq:D-tensor-equiv}
 T_y[QMQ^\top]=QT[M]
 \quad\text{for every symmetric }M.
\end{equation}
The coefficient formulas then give
\begin{equation}
\label{eq:D-coefficient-equivariance}
 B_{11,y}=QB_{11},\qquad
 B_{20,y}^{a}(H)=QB_{20}^{a}(H),\quad
 a\in\{\Scaf,\FedAvg,\mathrm{ctrl}\}.
\end{equation}
The simulation applies the same transformation to every noise draw.
Table~\ref{tab:num-structural-settings} lists the rotation and affine
comparison settings.

\begin{table}[tbp]
\centering\small
\setlength{\tabcolsep}{4pt}
\caption{Structural consistency checks. Corresponding initial states and
noise realizations are used within each comparison.}
\label{tab:num-structural-settings}
\begin{tabularx}{\linewidth}{@{}>{\raggedright\arraybackslash}p{0.22\linewidth}>{\raggedright\arraybackslash}X@{}}
\toprule
Check & Configuration \\
\midrule
Orthogonal rotation &
$Q=\left(\begin{smallmatrix}0.8660254037844386&-0.5\\0.5&0.8660254037844386\end{smallmatrix}\right)$ ($30^\circ$);
$H\in\{2,4\}$; $(N,\gamma)\in\{(16,1/16),(64,1/64)\}$;
4 paired runs per configuration, 16 paths in total \\
Affine translation &
$N\in\{8,32\}$; $H\in\{2,4,8\}$;
$a_c=\pm\rho u$, $\rho\in\{0,0.25,0.5,1\}$;
$\gamma=1/128$; 24 configurations of 2,000 rounds \\
\bottomrule
\end{tabularx}
\end{table}

\paragraph{Results and interpretation.}
The full-path residuals satisfy
\[
 \max_t\|x_t^{(y)}-Qx_t^{(x)}\|_\infty\le1.1491\times10^{-14},\qquad
 \max_{t,c}\|\xi_{t,c}^{(y)}-Q\xi_{t,c}^{(x)}\|_\infty\le1.5765\times10^{-14},
\]
with $\max_{t,h}\|\bar q_{t,h}^{(y)}-Q\bar q_{t,h}^{(x)}\|_\infty
\le5.4210\times10^{-17}$ and coefficient residuals below
$6\times10^{-16}$. All are within the $10^{-10}$ tolerance,
confirming coordinate consistency of the \Scaf{} implementation.

\subsubsection{Affine translation and update identities}
\label{app:num-affine-checks}

\paragraph{Design and rationale.}
The affine comparisons use balanced gradient offsets and the translated
initial controls prescribed by Proposition~\ref{prop:affine-shift}.
With common noise, the transformed states should agree at every round.
Zero initialization in both untransformed systems gives different
transients because the initial states are no longer translated counterparts.
Long-run $H=1$ checks separately assess arithmetic accumulation.

\paragraph{Results and interpretation.}
The affine residual reported in Section~\ref{sec:num-affine} reproduces
the exact translation to numerical precision. Short-run update-identity
checks also pass. Over longer runs, 48 of 132 $H=1$ arithmetic checks
exceed $10^{-10}$, with maximum vector discrepancy about
$7.54\times10^{-10}$. The shared-objective subset has 24 exceedances
among 60 checks and maximum coordinate discrepancy $4.33\times10^{-10}$.
Two additional paths using a stable endpoint-increment calculation reduce
vector discrepancies to about $10^{-12}$. This improvement is documented
for those two paths; the original full-grid outputs remain the reported
results.

\subsection{Real-data analyses and diagnostics}
\label{app:num-real-analyses}

\subsubsection{Objectives, partitions, and allocation estimands}
\label{app:supp-real}

\paragraph{Design and rationale.}
The shared-objective comparison isolates allocation effects, while disjoint
partitions introduce heterogeneous local objectives. We use the a9a
training and test sets from the LIBSVM collection.\footnote{\url{https://www.csie.ntu.edu.tw/~cjlin/libsvmtools/datasets/binary.html}}
Table~\ref{tab:num-real-settings} collects preprocessing, training, and
averaging settings. Each client minimizes
\begin{equation}
\label{eq:real-objective}
f_c(x)=\frac1{|D_c|}\sum_{i\in D_c}\{\log(1+e^{z_i^\top x})-y_i z_i^\top x\}+\frac{\lambda}{2}\|x\|^2.
\end{equation}
Equal client weights preserve the global empirical objective. The
shared-objective reference aggregate covariance has trace $0.001060$
after rounding; disjoint partitions can change reference covariances and
local objectives, including under IID assignment.

\begin{table}[tbp]
\centering\small
\setlength{\tabcolsep}{4pt}
\caption{Real-data settings. All training configurations use twelve
independent \Scaf{}--\FedAvg{} pairs. Windows within a run share its
trajectory. The two longer-run pilots are analyzed separately.}
\label{tab:num-real-settings}
\begin{tabularx}{\linewidth}{@{}>{\raggedright\arraybackslash}p{0.30\linewidth}>{\raggedright\arraybackslash}X@{}}
\toprule
Component & Specification \\
\midrule
Training data &
32,512 of 32,561 rows selected by a fixed random permutation to allow equal disjoint partitions \\
Features and penalty &
Each 123-feature row divided by $\max\{1,\|z_i\|\}$, then augmented with an intercept;
$y_i\in\{0,1\}$; all 124 parameters regularized with $\lambda=0.01$ \\
Client objectives &
Shared: all clients use the same data; disjoint IID: random partition;
label skew: two label-sorted shards per client.
Disjoint regimes use one primary and two alternative partitions \\
Minibatches and base step &
Sampling with replacement; $L_{\rm sample}=0.51$;
$\gamma_b=0.1/L_{\rm sample}=0.1960784314$ \\
Fixed-budget allocations &
$Nm=256$; $(N,m)\in\{(8,32),(32,8),(128,2)\}$;
$H=4$, $\gamma=\gamma_b$ \\
Additional shared-objective runs &
Three step-size/local-step variants (Table~\ref{tab:supp-policy}) and three $H=1$ checks;
nine shared-objective configurations in total \\
Averaging windows &
$K\in\{1,2,4,8,16,32,64\}$ final batches per run;
primary configuration $H=4,N=128,m=2,\gamma_b$ \\
Primary averaging horizon &
1,280 rounds per batch; 6,400 burn-in rounds;
longest window 81,920 rounds, ending at round 88,320;
90,439,680 example-gradient evaluations per algorithm \\
Longer-run pilots &
Two runs; 25,600 burn-in rounds and 128 batches of 1,280 rounds;
189,440 rounds per run \\
\bottomrule
\end{tabularx}
\end{table}

The reference Newton solve has gradient norm $1.253\times10^{-16}$ and
strong-convexity parameter-error bound $1.253\times10^{-14}$; a separately
initialized solver differs by $8.43\times10^{-15}$. These floating-point
residuals quantify numerical accuracy of the reference used to evaluate
$A_{\rm ref}$, $T_{\rm ref}$, and $\Sigma_{\rm ref}$.
The projection direction in Equation~\eqref{eq:real-projection} is fixed
before training.

Squared mean displacement is estimated by $\widehat U_K$ in
Equation~\eqref{eq:e4-estimators} at the complete tail. Raw-checkpoint
MSE comes from separately accumulated squared errors; its variance estimate
subtracts the corresponding cross-run squared-mean estimate. Negative
sampling-corrected values are retained. The nominal jackknife-$t$ intervals
and adjusted contrast families are described in Table~\ref{tab:num-inference}.

\paragraph{Results and interpretation.}
Table~\ref{tab:supp-real-full} adds raw error and held-out loss results
to the main allocation study, and supplies the disjoint-partition results.
In the label-skew case at $N=32,H=4$, deterministic \FedAvg{} converges
to a point about $0.08365$ from $x_{\rm ref}$. The algorithm difference
therefore includes a substantial drift-correction effect in addition to
stochastic mean displacement.

The pointwise projection half-width target $10^{-4}$ is unmet in 32 of
35 real-data configurations, including the main shared-objective
allocations. The separately estimated allocation contrasts nevertheless
resolve their signed differences. Precision of an individual mean and
precision of a paired contrast are distinct features of the experiment.

\begin{table}[tbp]
\centering\small
\caption{Additional finite-tail allocation results at $H=4$, $\gamma_b$,
and $Nm=256$. The upper panel supplements Figure~\ref{fig:real-allocation}
with error and loss measurements; the lower panel reports disjoint
partitions. Projection intervals are nominal 95\% intervals.
Raw-checkpoint MSE is multiplied by $10^3$; held-out loss is averaged
over the per-run tail-average models. S: SCAFFOLD; F: FedAvg.}
\label{tab:supp-real-full}
\begin{tabular}{@{}crrrr@{}}
\toprule
\multicolumn{5}{c}{Shared objective: additional error and loss measurements} \\
$N$ & MSE (S) & MSE (F) & Test loss (S) & Test loss (F) \\
\midrule
8 & $2.3511$ & $2.3509$ & $0.424981$ & $0.424983$ \\
32 & $2.3629$ & $2.3618$ & $0.424961$ & $0.424970$ \\
128 & $2.3745$ & $2.3655$ & $0.424889$ & $0.424922$ \\
\bottomrule
\end{tabular}

\medskip
\begin{tabular}{@{}ccrlrr@{}}
\toprule
$N$ & Alg. & Projection & Nominal interval & MSE & Test loss \\
\midrule
\multicolumn{6}{l}{\emph{Disjoint IID partition}} \\
8 & S & $0.000341$ & $[0.000213,0.000468]$ & $2.3458$ & $0.424978$ \\
8 & F & $0.000268$ & $[0.000140,0.000396]$ & $2.3454$ & $0.424981$ \\
32 & S & $0.000880$ & $[0.000638,0.001123]$ & $2.3661$ & $0.424963$ \\
32 & F & $0.000472$ & $[0.000230,0.000715]$ & $2.3643$ & $0.424977$ \\
128 & S & $0.003553$ & $[0.003393,0.003714]$ & $2.3642$ & $0.424884$ \\
128 & F & $0.001994$ & $[0.001830,0.002158]$ & $2.3507$ & $0.424936$ \\
\midrule
\multicolumn{6}{l}{\emph{Disjoint label-skew partition}} \\
8 & S & $0.000243$ & $[0.000050,0.000436]$ & $2.2526$ & $0.424984$ \\
8 & F & $-0.039298$ & $[-0.039491,-0.039104]$ & $5.6791$ & $0.426589$ \\
32 & S & $0.000352$ & $[0.000143,0.000561]$ & $2.2180$ & $0.424981$ \\
32 & F & $-0.055891$ & $[-0.056096,-0.055686]$ & $9.0666$ & $0.427274$ \\
128 & S & $0.001256$ & $[0.001096,0.001416]$ & $2.1956$ & $0.424956$ \\
128 & F & $-0.062027$ & $[-0.062182,-0.061872]$ & $10.7951$ & $0.427548$ \\
\bottomrule
\end{tabular}
\end{table}

\subsubsection{Run-length sensitivity}
\label{app:supp-audit}

\paragraph{Design and rationale.}
First/last-half comparisons assess variation within the recorded tails.
They use the same 64-batch trajectories and the $2\times69$ adjustment
across controlled and real-data projected contrasts. In the shared-objective
real-data study, the primary observable is the direct mean itself.

\paragraph{Results and interpretation.}
The adjusted first/last-half contrast intervals contain zero. The endpoint
allocation contrasts in the first and last halves are $0.003098$ and
$0.003129$ for \Scaf{}, and $0.002053$ and $0.002092$ for \FedAvg{}.
The similar values support stability of the allocation comparison within
the sampled tails. The stronger temporal-equivalence criterion is assessed
separately below.

\subsubsection{Finite-tail model averaging}
\label{app:e4-averaging}

\paragraph{Design and rationale.}
To quantify the components in Equation~\eqref{eq:e4-finite-window-mse},
set $e_{r,K}^{\mathcal A}=Y_{r,K}^{\mathcal A}-x_{\rm ref}$ and
$\bar e_K=R^{-1}\sum_r e_{r,K}$ with $R=12$.
Suppressing the algorithm superscript,
\begin{equation}
\label{eq:e4-estimators}
\widehat U_K=\frac{\sum_{r\ne s}e_{r,K}^{\top}e_{s,K}}{R(R-1)},\qquad
\widehat V_K=\frac{\sum_r\|e_{r,K}-\bar e_K\|^2}{R-1},\qquad
\widehat M_K=\frac1R\sum_r\|e_{r,K}\|^2=\widehat U_K+\widehat V_K.
\end{equation}
Cross-products between independent runs estimate the squared mean
displacement, while between-run deviations estimate the variance of the
time-averaged outputs. The estimators retain negative values of
$\widehat U_K$. Training and test losses are evaluated on each run's own
$Y_{r,K}$ and then averaged across runs.

The primary, secondary, and additional configurations use the schedules
in Table~\ref{tab:num-real-settings}. The $N=8,32$ allocations are
secondary; six further configurations examine step size, local steps,
and $H=1$. Bootstrap resampling uses one common run-index vector across
algorithms, configurations, and overlapping windows. Separate uncertainty
families are listed in Table~\ref{tab:num-inference}; all condition on
the fixed dataset and recorded finite-tail protocol.

\paragraph{Results and interpretation.}
Table~\ref{tab:e4-primary} gives absolute errors and losses complementing
the main-text error shares. For $b_S=b_F+\delta$, the difference in squared
means obeys $\|b_S\|^2-\|b_F\|^2=2b_F^\top\delta+\|\delta\|^2$.
The sampling-corrected decomposition is
$\Delta\widehat U=2(2.57965\times10^{-6})+1.21856\times10^{-6}$,
showing the contribution of both the cross term and the squared algorithmic
displacement. The total-MSE contrast corresponds to a $58.7\%$ relative
difference. For scale, $f(x_{\rm ref})=0.4758413$ and the held-out contrast
has magnitude about $0.00769\%$ of reference test loss.

\begin{table}[t]
\centering\small
\caption{Additional longest-window model-averaging statistics at $K=64$,
with approximate simultaneous 95\% primary-family intervals.
Parameter errors and training gaps are multiplied by $10^6$;
test loss is unscaled. The component shares and total-MSE contrast
are reported in Section~\ref{sec:num-averaging}.}
\label{tab:e4-primary}
\begin{tabular}{@{}lrr@{}}
\toprule
Quantity & SCAFFOLD & FedAvg \\
\midrule
$\widehat M_K$ & $17.267\ [15.383, 19.151]$ & $10.878\ [9.459, 12.297]$ \\
Train objective gap & $0.725\ [0.630, 0.819]$ & $0.361\ [0.295, 0.426]$ \\
Held-out log loss & $0.424889\ [0.424879, 0.424899]$ & $0.424922\ [0.424912, 0.424932]$ \\
\midrule
Paired difference & \multicolumn{2}{c}{Estimate [simultaneous interval]} \\
$\Delta\widehat U_K$ & \multicolumn{2}{c}{$6.378\ [5.848, 6.908]$} \\
$\Delta\widehat V_K$ & \multicolumn{2}{c}{$0.011\ [-0.036, 0.058]$} \\
\midrule
\multicolumn{3}{l}{Additional intervals for the loss contrasts reported in the main text:} \\
$\Delta$ train gap & \multicolumn{2}{c}{$[0.336, 0.393]$} \\
$\Delta$ held-out log loss & \multicolumn{2}{c}{$[-0.00003308, -0.00003231]$} \\
\bottomrule
\end{tabular}
\end{table}

At $N=32$ and $8$, the longest-window \Scaf{}/\FedAvg{} mean-displacement
shares are $14.97\%/7.77\%$ and $1.61\%/0.97\%$, respectively.
Their squared-mean intervals give weaker evidence against zero than the
primary allocation. Together with the main result, these comparisons show
that the error share depends on allocation as well as averaging length.

\paragraph{Temporal diagnostics.}
The predeclared equivalence test requires both a $10^{-3}$ vector-norm
margin and a $10^{-4}$ projected-drift margin. None of 144 checks meets
both requirements. In the primary \Scaf{} comparison, the first/last
32-batch drift norm is $0.0013633$, with upper bound $0.0054372$ and
projected interval $[-0.0005438,0.0007484]$. The longest window inherits
the 32-batch check because a separate equal-length 64-batch anchor is
unavailable. Thus temporal equivalence remains unresolved.

For \Scaf{}, $\widehat U_K\geq\widehat V_K$ at $K=32,64$, with a
positive lower bound for $U-V$ under both interval procedures only at
$K=64$. The prespecified crossover requirement---two consecutive supported
windows together with temporal equivalence---is unmet. \FedAvg{} has
no point-estimate crossover on this grid. The protocol's $10\%$
parameter-MSE-share threshold is a researcher-chosen reference.
The two longer-run pilots in Table~\ref{tab:num-real-settings} leave
these temporal questions unresolved. The established result is the
finite-tail error decomposition reported in Section~\ref{sec:num-averaging}.

\subsubsection{Post-hoc directional comparison}
\label{app:e5-bridge}

\paragraph{Design and rationale.}
To contextualize the alignment in Table~\ref{tab:e5-bridge}, we compare
the analytical direction with the fixed reference-parameter direction
and the unpreconditioned vector $-T_{\rm ref}[\Sigma_{\rm ref}]$.
Window, step-size, and local-step comparisons assess how alignment and
amplitude change within the existing runs. The zero-vector proxy supplies
an amplitude baseline. Sampling procedures are listed in
Table~\ref{tab:num-inference}; scalar and vector families each have their
own 95\% calibration.

\paragraph{Results and interpretation.}
At $K=64$, reference-parameter cosines are $0.8953$--$0.9087$ and
unpreconditioned-direction cosines are $0.6590$--$0.6929$.
Both are lower than the analytical proxy values in the main table; the
zero-vector proxy also has larger residuals in all six configurations.
These comparisons make the geometric agreement specific to the
Hessian-preconditioned curvature--noise direction among the tested references.

For the fixed-$H,N,m$ sequence, adjacent log slopes of the observed
projection are $1.8553$ and $1.8907$. Across shorter averaging windows,
the cosine falls as low as $0.872197$. These results describe finite-grid
and window sensitivity. Allocation comparisons change $N$ and $m$
together, preserving total sampling budget while changing local covariance.

For the $H=4/H=2$ ratio in Section~\ref{sec:num-directional}, the
exploratory simultaneous interval is $[2.7415,3.1125]$, the jackknife
interval is $[2.7101,3.1439]$, and the projection-band transform is
$[2.6928,3.1919]$. All exclude $3.5$ and include the ordinary local-update
ratio $3$. The results resolve a departure from the plug-in local-step
scaling; attribution of that departure remains open. The $H=1$ residuals
are approximately $10^{-11}$--$10^{-10}$, consistent with the expected
zero-proxy numerical identity.

The 54 early/late paired contrasts assess temporal variation of the
algorithm difference. The model-averaging diagnostics in
Appendix~\ref{app:e4-averaging} assess each algorithm separately, so
the two comparisons address different temporal quantities. The proxy
agreement remains a descriptive geometric result for the recorded
finite-tail means.

\end{document}